\def\arxivversion{1}

\documentclass{article} 
\usepackage{iclr2027_conference,times}

\usepackage{amsmath,amsfonts,bm}

\def\eqref#1{equation~\ref{#1}}

\def\1{\bm{1}}

\DeclareMathAlphabet{\mathsfit}{\encodingdefault}{\sfdefault}{m}{sl}
\SetMathAlphabet{\mathsfit}{bold}{\encodingdefault}{\sfdefault}{bx}{n}

\usepackage{graphicx}
\usepackage{etoolbox}
\makeatletter
\@ifundefined{Gread@transgrouptrue}{}{%
  \patchcmd{\Gread@pdf}{\Gread@transgrouptrue}{\Gread@transgroupfalse}{}%
    {\PackageError{arxiv}{Could not preserve PDF figure compositing}{Check the graphics driver.}}}
\makeatother
\usepackage{float}
\usepackage{capt-of}
\usepackage{flafter}
\usepackage[hypertexnames=false]{hyperref}
\usepackage{url}
\usepackage{algorithm}
\usepackage{algpseudocode}
\usepackage{amsthm}
\usepackage{booktabs}
\usepackage{multirow}
\usepackage{wrapfig}
\usepackage{needspace}
\usepackage{xcolor}

\newif\ifreviewchanges
\reviewchangesfalse
\definecolor{reviewblue}{RGB}{0,77,166}
\newcommand{\reviewcolor}{\ifreviewchanges\color{reviewblue}\fi}
\DeclareRobustCommand{\review}[1]{\leavevmode{\reviewcolor #1}}
\newcommand{\reviewtablerows}[1]{\csname @@input\endcsname{#1}}

\AddToHook{env/figure/begin}{%
  \setlength{\intextsep}{8pt}%
  \setlength{\abovecaptionskip}{6pt}%
  \setlength{\belowcaptionskip}{2pt}}

\newtheorem{theorem}{Theorem}

\title{Value as Memory: Online Value-Guided Post-training for General Robot Policies}

\ifdefined\arxivversion
\iclrfinalcopy
\usepackage{etoolbox}
\makeatletter
\patchcmd{\@maketitle}{\ificlrfinal}{\vskip 3pt\ificlrfinal}{}%
  {\PackageError{arxiv}{Could not adjust title-author spacing}{Check the title template.}}
\patchcmd{\@maketitle}{\vskip 0.3in minus 0.1in}{\vskip 0.14in minus 0.05in}{}%
  {\PackageError{arxiv}{Could not adjust author-abstract spacing}{Check the title template.}}
\makeatother
\author{%
\makebox[\dimexpr\textwidth-2\tabcolsep\relax][c]{%
\begin{tabular}{@{}c@{}}
Wenye Yu\textsuperscript{1,2,*}\quad
Yijin Chen\textsuperscript{1,*}\quad
Kaiyi Xiong\textsuperscript{1,*}\quad
Wenbo Han\textsuperscript{1}\quad
Wendi Chen\textsuperscript{1,2}\\[2pt]
Yang Jin\textsuperscript{1}\quad
Han Xue\textsuperscript{1}\quad
Jun Lv\textsuperscript{1,3}\quad
Cewu Lu\textsuperscript{1,2,3,\ensuremath{\dagger}}\quad
Chuan Wen\textsuperscript{1,\ensuremath{\dagger}}\\[4pt]
\normalfont\textsuperscript{1}Shanghai Jiao Tong University\\
\normalfont\textsuperscript{2}Shanghai Innovation Institute\quad
\normalfont\textsuperscript{3}Noematrix Ltd.
\end{tabular}}%
}

\makeatletter
\newcommand{\arxivtitlefootnotes}{%
  \fancyhead{}%
  \renewcommand{\headrulewidth}{0pt}%
  \renewcommand{\footrulewidth}{0pt}%
  \begingroup
    \renewcommand{\thefootnote}{}%
    \long\def\@makefntext##1{\noindent##1}%
    \footnotetext[0]{%
      \textsuperscript{*}Equal contribution.\quad
      \textsuperscript{\ensuremath{\dagger}}Corresponding authors.\\
      Project page: \url{https://virlus.github.io/vam/}}%
  \endgroup
  \setcounter{footnote}{0}}
\makeatother

\else

\author{Antiquus S.~Hippocampus, Natalia Cerebro \& Amelie P. Amygdale \thanks{ A footnote may provide further information
about an author (webpage, alternative address)---\emph{not} acknowledgments for
funding agencies.  Funding acknowledgements go at the end of the paper.} \\
Department of Computer Science\\
Cranberry-Lemon University\\
Pittsburgh, PA 15213, USA \\
\texttt{\{hippo,brain,jen\}@cs.cranberry-lemon.edu} \\
\And
Ji Q. Ren \& Yevgeny LeNet \\
Department of Computational Neuroscience \\
University of the Witwatersrand \\
Joburg, South Africa \\
\texttt{\{robot,net\}@wits.ac.za} \\
\AND
Coauthor \\
Affiliation \\
Address \\
\texttt{email}
}

\fi

\begin{document}

\maketitle
\ifdefined\arxivversion
  \arxivtitlefootnotes
\fi

\begin{abstract}
Real-world post-training of generalist robot policies faces two challenges: partial observability makes action selection ambiguous, while offline post-training methods delay feedback from deployment experience.
We present Value as Memory (VaM), a framework that addresses these challenges through value caching and online real-world reinforcement learning (RL).
Value caching conditions actions on recent value predictions as task-progress memory, while online updates enable timely policy improvement by combining offline demonstrations, off-policy interventions, and on-policy rollouts for sample-efficient adaptation.
As a result, VaM transitions from offline RL warmup to online autonomous learning with sparse rewards.
Across long-horizon tasks involving deformable objects and contact-rich manipulation, VaM yields a mean absolute success-rate gain of $29.4\%$ with value caching, while achieving a $32.8\%$ gain with online autonomous learning over matched offline training.
\review{The evaluated online schedule performs $8\times$ as many policy iterations under the same gradient-step budget, with more than $2\times$ lower recorded time per gradient step with concurrent data collection and training.}

\end{abstract}

\setlength{\parskip}{3pt}
\setlength{\abovedisplayskip}{5pt plus 1pt minus 1pt}
\setlength{\belowdisplayskip}{5pt plus 1pt minus 1pt}
\begin{figure}[H]
    \centering
    \includegraphics[width=\linewidth]{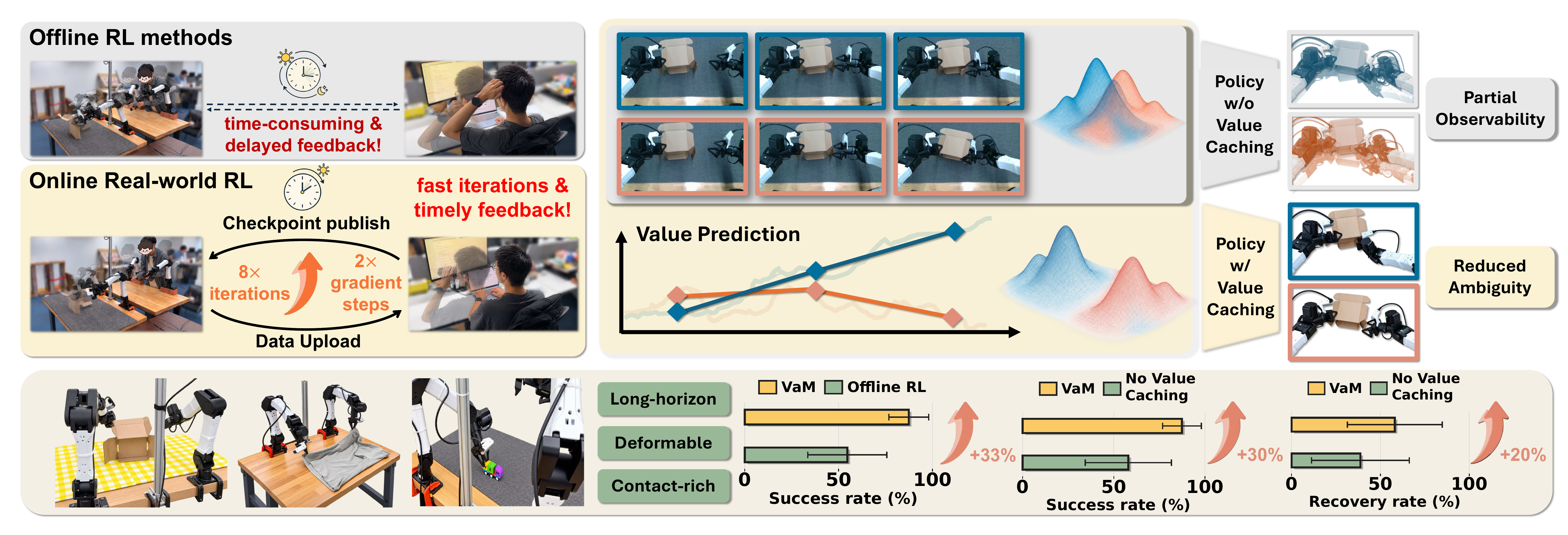}
    \caption{VaM uses online real-world RL and value caching to boost sample efficiency and mitigate partial observability in long-horizon, deformable-object, and contact-rich manipulation.}
    \label{fig:teaser}
\end{figure}

\section{Introduction}

Real-world reinforcement learning (RL) enables generalist robot policies to learn from physical interaction and improve on deployment tasks \citep{chen2025conrft,physicalintelligence2025pi06}.
Initialized from offline expert demonstrations, a policy aims to maximize cumulative task rewards \citep{luo2024serl,frans2025diffusion} and optionally learns from human-intervention data during deployment \citep{luo2025hilserl,wang2026learningwhiledeploying}.

However, current real-world RL methods face limitations in data reuse and policy-update frequency.
On-policy updates primarily use fresh rollouts, limiting the reuse of historical experience \citep{chen2025pirl}.
Iterated offline RL separates data collection from optimization, delaying feedback from newly encountered states \citep{physicalintelligence2025pi06}.
Besides, partial observability poses another challenge.
With limited observation histories, visually similar states may require different actions.
Meanwhile, value functions provide compact task-progress estimates, but typically serve as filtering scores or optimization weights \citep{physicalintelligence2025pi06,frans2025diffusion}.
We therefore ask whether value estimates on observation histories can augment policy inputs to mitigate partial observability, and how an online RL system can exploit a learned value function to incorporate offline demonstrations, off-policy data, and on-policy rollouts for joint adaptation of the policy and value functions with high sample efficiency and rapid closed-loop feedback.

To this end, we introduce \textbf{Value as Memory (VaM)}, an offline-to-online RL system with two key features.
\textbf{First}, we propose value caching, which extends value predictions beyond training supervision by supplying their recent history as an additional policy input.
Alongside visual and proprioceptive histories, this cache indicates whether estimated progress improves, stalls, or declines, helping distinguish states requiring different actions.
\textbf{Second}, VaM uses a hybrid training data scheme that combines offline demonstrations, off-policy trajectories, and current on-policy rollouts for sample-efficient joint adaptation of policy and value models.
Our training regime progresses from offline RL warmup through human-assisted collection to autonomous learning with sparse rewards.
Across real-world tasks involving deformable objects and contact-rich manipulations, VaM improves mean absolute success rate by $29.4\%$ over its variant without value caching.
\review{It gains $32.8\%$ in mean absolute success rate over its offline RL counterpart under matched training budgets.}
\review{The evaluated online schedule performs $8\times$ as many policy iterations under the same gradient-step budget, with more than $2\times$ lower recorded time per gradient step with concurrent data collection and training.}
These results support value caching for mitigating partial observability and our online RL system with a hybrid training data scheme for improving sample efficiency and policy iteration frequency.

Our contributions are twofold:
\begingroup
\setlength{\topsep}{0pt}
\setlength{\partopsep}{0pt}
\setlength{\parsep}{0pt}
\setlength{\itemsep}{1pt}
\begin{enumerate}
    \item We propose value caching to condition policies on recent task-progress estimates, with theoretical analysis and real-world ablations supporting its role in mitigating action ambiguity under partial observability.
    \item We develop an online RL framework that reuses offline demonstrations, off-policy trajectories, and current on-policy rollouts for joint policy and value adaptation, improving sample efficiency and policy iteration frequency.
\end{enumerate}
\endgroup

\section{Preliminaries}
\label{sec:preliminaries}

\subsection{Problem Formulation}

We model iterative robot post-training as a Markov Decision Process (MDP) $\mathcal{M}=(\mathcal{S},\mathcal{A},P,r,\gamma)$, where $\mathcal{S}$ and $\mathcal{A}$ are the state and action spaces, $P(s_{t+1}\mid s_t,a_t)$ is the transition kernel, $r(s_t,a_t)$ is the reward function, and $\gamma\in[0,1)$ is the discount factor.
The policy $\pi_\theta(a_{t:t+H}\mid s_t)$ has parameters $\theta$ and predicts the action sequence $a_{t:t+H}=(a_t,\ldots,a_{t+H-1})$, where $H$ is the prediction horizon.
We let $\mathcal{D}_{\mathrm{offline}}=\{\tau_{\mathrm{offline},n}\}_{n=1}^{N}$ denote $N$ successful expert demonstrations and $\pi_{\theta^{(0)}}$ the policy initialized from them.
Post-training then runs for $M$ rounds.
At each round $m\in\{0,\ldots,M-1\}$, we deploy $\pi_{\theta^{(m)}}$ to collect rollout and human-intervention trajectories $\mathcal{D}_{\mathrm{online}}^m$ and use the accumulated data to obtain $\pi_{\theta^{(m+1)}}$.
With transition reward $r_t=r(s_t,a_t)$, our objective is to maximize the expected discounted cumulative reward $\mathbb{E}_{\pi_\theta}\!\left[\sum_{t=0}^{\infty}\gamma^t r_t\right]$.

\subsection{Distributional Value Learning}

We learn a distributional value function to estimate task progress and provide the scalar advantages used for policy improvement.
Following previous literature on distributional value learning \citep{bellemare2017distributional,li2025grrl,physicalintelligence2025pi06}, the value network with parameters $\psi$ predicts a categorical distribution $p_\psi(z_k\mid s)$ on $K\geq2$ uniformly spaced atoms $\mathcal{Z}=\{z_k\}_{k=1}^{K}\subset[0,1]$, where $z_k=(k-1)/(K-1)$ and the scalar value is $V_\psi(s)=\sum_{k=1}^{K}z_k p_\psi(z_k\mid s)$.
We initialize the value expert with normalized demonstration progress, then adapt it through temporal-difference (TD) learning with sparse outcome rewards.
For a successful demonstration $\tau_{\mathrm{offline}}\in\mathcal{D}_{\mathrm{offline}}$ with $T$ timesteps, the initialization target $q_t^{\mathrm{init}}$ is a normalized Gaussian label over $\mathcal{Z}$ centered at $t/T$ with variance $\sigma_{\mathrm{init}}^2>0$.
The cross-entropy (CE) initialization objective is
\begin{equation}
    \mathcal{L}_{\mathrm{init}}(\psi)
    =\mathbb{E}_{\tau_{\mathrm{offline}},t}\!\left[\operatorname{CE}\!\left(q_t^{\mathrm{init}},p_\psi(\cdot\mid s_t)\right)\right],
    \quad
    \operatorname{CE}(q,p)=-\sum_{k=1}^{K}q_k\log p_k.
    \label{eq:progress_initialization}
\end{equation}
For TD learning, only successful terminal transitions receive $+1$, with zero rewards otherwise.
After each gradient step, the target parameters $\bar\psi$ follow the exponential moving average (EMA) \mbox{$\bar\psi\leftarrow(1-\beta)\bar\psi+\beta\psi$}, with $\beta\in(0,1]$.
We define $R_{t:t+H}=\sum_{h=0}^{H-1}\gamma^h r_{t+h}$ and a terminal indicator $d_t\in\{0,1\}$ for the $H$-step segment.
\review{The TD target $q_t^{\mathrm{TD}}$ moves each mass $p_{\bar\psi}(z_k\mid s_{t+H})$ to $R_{t:t+H}+\gamma^H(1-d_t)z_k$ before linear projection onto $\mathcal{Z}$ (detailed derivation in Appendix~\ref{app:training_details}).}
We train the value expert with
\begin{equation}
    \mathcal{L}_{\mathrm{TD}}(\psi)
    =\mathbb{E}_{t}\!\left[\operatorname{CE}\!\left(q_t^{\mathrm{TD}},p_\psi(\cdot\mid s_t)\right)\right].
    \label{eq:td_value_learning}
\end{equation}
The TD-adapted value expert gives the scalar advantage estimate $\widehat{A}_t=R_{t:t+H}+\gamma^H(1-d_t)V_\psi(s_{t+H})-V_\psi(s_t)$.

\subsection{Advantage-conditioned Policy Learning}
\label{sec:advantage_conditioned_policy}

The action expert learns $\pi_\theta$ by flow matching, with an optional advantage condition to distinguish action quality.
For a sample $(s_t,a_{t:t+H})\sim\mathcal{D}_{\mathrm{BC}}$, where $\mathcal{D}_{\mathrm{BC}}$ is the behavior-cloning dataset and $a_{t:t+H}$ is an $H$-step action sequence, we draw noise $\epsilon\sim\mathcal{N}(0,I)$ with identity covariance $I$ and flow timestep $\lambda\sim\mathcal{U}(0,1)$, and set $a_\lambda=(1-\lambda)\epsilon+\lambda a_{t:t+H}$ \citep{black2025pi05}.
The velocity field $v_\theta$ shares the policy parameters $\theta$ and takes an optional condition $c$, where $c=\emptyset$ denotes behavior cloning.
With the expectation taken over the sampled data, noise, and flow timestep, the flow-matching objective and its behavior-cloning instance are
\begin{equation}
    \mathcal{L}_{\mathrm{FM}}(\theta;c)=\mathbb{E}\!\left[\left\lVert v_\theta(a_\lambda,\lambda,s_t,c)-(a_{t:t+H}-\epsilon)\right\rVert_2^2\right],\quad \mathcal{L}_{\mathrm{BC}}(\theta)=\mathcal{L}_{\mathrm{FM}}(\theta;\emptyset).
    \label{eq:flow_matching_objectives}
\end{equation}
\begingroup
\makeatletter
\edef\@currentlabel{\p@equation\theequation}
\label{eq:flow_matching_bc}
\label{eq:advantage_conditioned_flow_matching}
\makeatother
\endgroup
For advantage conditioning, CFGRL provides a policy-improvement guarantee when the optimality probability increases monotonically with the reference-policy advantage \citep{frans2025diffusion,physicalintelligence2025pi06}.
For online rollout data $\mathcal{D}_{\mathrm{online}}^m$ collected at round $m$, we set the threshold $\xi_m^\rho=\operatorname{Quantile}_{\rho}(\{\widehat{A}_i:(s_i,a_{i:i+H})\in\mathcal{D}_{\mathrm{online}}^m\})$ with percentile parameter $\rho\in(0,1)$.
The binary advantage condition is $o_t=\mathbb{I}[\widehat{A}_t\geq\xi_m^\rho]\in\{0,1\}$.
During training, we randomly replace $o_t$ with $\emptyset$, yielding $o_t\in\{\emptyset,0,1\}$, and minimize $\mathcal{L}_{\mathrm{FM}}(\theta;o_t)$.
At inference, $o_t=1$ requests high-advantage actions, and classifier-free guidance with weight $w\geq1$ combines the conditional and unconditional velocity fields \citep{ho2022classifierfree}.

\section{Method}
\label{sec:method}

\subsection{Method Overview}

VaM uses a distributional value expert to supply recent progress predictions to the action expert through a value cache and advantage labels for policy post-training.
Training proceeds through four stages: offline demonstration initialization, offline RL warmup, Online DAgger, and Learn from Experience (LfE).
The value and action experts are updated in parallel during online RL.
\review{Figure~\ref{fig:method_overview} summarizes this training pipeline and the value caching architecture.}

\begin{figure}[!t]
    \centering
    \includegraphics[width=\textwidth]{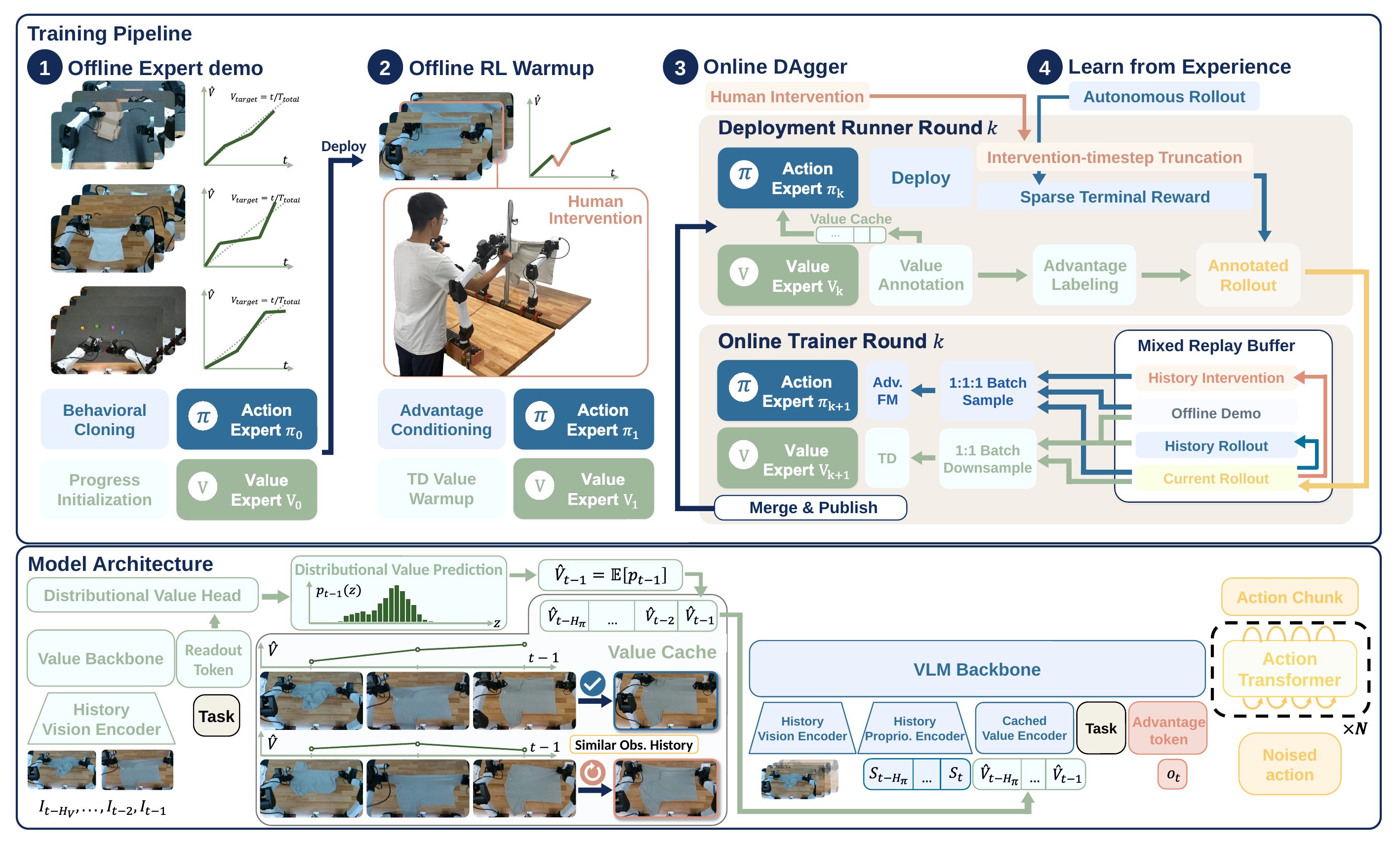}
    \caption{\textbf{VaM training pipeline and model architecture.}
    The top panel summarizes the four-stage offline-to-online training pipeline.
    The bottom panel shows how the distributional value expert turns previously computed value predictions into the value cache that conditions the action expert.}
    \label{fig:method_overview}
\end{figure}

\subsection{Value Caching and Soft Aliasing Analysis}
\label{sec:value_caching}

\subsubsection{\review{Value Caching}}

The value cache gives the action expert explicit recent progress information without delaying its current action.
At deployment round $k$, we fix the value snapshot $V_{\psi^{(k-1)}}$ obtained at the end of the preceding training round and use the expectation of its categorical return distribution as the scalar progress signal.
At timestep $t$, we denote the action expert's visual-observation history by $h_t^\pi$, the value expert's visual-observation history by $h_t^V$, and the action expert's proprioceptive-observation history by $s_t$.
The fixed value snapshot predicts $\widehat V_{\psi^{(k-1)},t}=V_{\psi^{(k-1)}}(h_t^V)$.
Value prediction and action generation run asynchronously, so the value prediction for timestep $t$ is unavailable when the action expert generates $a_{t:t+H}$.
Therefore, we let $H_\pi\geq1$ denote the action expert's observation horizon and define the value cache $c_t^{(k-1)}$ and the action expert's complete observation $x_t^\pi$ as
\begin{equation}
    \begin{aligned}
    c_t^{(k-1)}
    &=
    \left(\widehat V_{\psi^{(k-1)},t-H_\pi},\ldots,\widehat V_{\psi^{(k-1)},t-1}\right),
    \\
    x_t^\pi
    &=
    \left(h_t^\pi,s_t,c_t^{(k-1)}\right).
    \end{aligned}
    \label{eq:value_cache}
\end{equation}
The cache summarizes recent task progress and, when $H_\pi\geq2$, whether progress is improving, stalled, or declining, helping distinguish visually similar states that require different optimal actions.

\subsubsection{Why Value Caching Eases Soft Aliasing}
\label{sec:soft_aliasing_analysis}

Value caching targets soft aliasing, in which visually similar states require different high-advantage action distributions.
We formalize how augmenting the policy input with the value cache can relax a representation-induced lower bound on policy prediction error.
We consider two states indexed by $i\in\{1,2\}$ within a fixed deployment round $k$.
Equation~\ref{eq:value_cache} gives their value caches $c_i^{(k-1)}$ and complete action observations $x_i^\pi=(h_i^\pi,s_i,c_i^{(k-1)})$.
We let $z_i^{\mathrm{obs}}$ denote the representation of the visual-proprioceptive components $(h_i^\pi,s_i)$.
For state $i$, $\mu_i$ denotes the distribution of training action chunks assigned the high-advantage label defined in Sec.~\ref{sec:advantage_conditioned_policy}.
We measure the difference between action distributions with the Wasserstein distance $W_2$.
We represent the complete action observation $x_i^\pi$ as $z_i^\pi=(z_i^{\mathrm{obs}},c_i^{(k-1)})$.
We define
\begin{equation}
    \begin{aligned}
    \varepsilon_{\mathrm{obs}}
    &=\|z_1^{\mathrm{obs}}-z_2^{\mathrm{obs}}\|_2,
    \\
    \varepsilon_\pi
    &=\|z_1^\pi-z_2^\pi\|_2
    =\sqrt{\varepsilon_{\mathrm{obs}}^2+\|c_1^{k-1}-c_2^{k-1}\|_2^2},
    \\
    \Delta_a
    &=W_2(\mu_1,\mu_2),
    \end{aligned}
    \label{eq:value_cache_geometry}
\end{equation}
Soft aliasing occurs when $\varepsilon_{\mathrm{obs}}$ is small but $\Delta_a$ is large: the policy sees similar inputs even though the desired actions differ.
We assume that the policy is locally $L$-Lipschitz with respect to its input representation: for any two representations $z$ and $z'$ in the relevant neighborhood,
\begin{equation}
    W_2\!\left(\pi_\theta(\cdot\mid z,o_t=1),\pi_\theta(\cdot\mid z',o_t=1)\right)
    \leq L\|z-z'\|_2.
    \label{eq:local_lipschitz}
\end{equation}

\begin{theorem}[Local soft-aliasing bound]
\label{thm:soft_aliasing}
For representation $r\in\{\mathrm{obs},\pi\}$, $z_i^r$ denotes corresponding input for state $i$ and $\varepsilon_r$ its pairwise distance in Eq.~\ref{eq:value_cache_geometry}.
Any locally $L$-Lipschitz policy must incur
\begin{equation}
    \max_{i\in\{1,2\}}
    W_2\!\left(\pi_\theta(\cdot\mid z_i^r,o_t=1),\mu_i\right)
    \geq
    \frac{1}{2}(\Delta_a-L\varepsilon_r)_+,
    \label{eq:soft_aliasing_bound}
\end{equation}
where $(x)_+=\max(x,0)$.
In particular, the lower bound is strictly positive whenever $\Delta_a>L\varepsilon_r$, so no locally $L$-Lipschitz policy can match both target action distributions exactly.
If $L\varepsilon_{\mathrm{obs}}<\Delta_a\leq L\varepsilon_\pi$, an observation-only policy cannot match both target action distributions while satisfying the local Lipschitz condition, whereas exact matching at $z_1^\pi$ and $z_2^\pi$ does not violate this constraint.
\end{theorem}

The formal setup and proof are given in Appendices~\ref{app:soft_aliasing_setup} and~\ref{app:soft_aliasing_proof}.
Because $\varepsilon_\pi\geq\varepsilon_{\mathrm{obs}}$, the value cache can weaken the aliasing lower bound by making the two states easier to separate.

\subsection{Continual Offline-to-Online RL System}

Algorithm~\ref{alg:training_pipeline} in Appendix~\ref{sec:appendix} summarizes the four-stage training loop, with data and objective details in Appendix~\ref{app:training_details}.

\subsubsection{Training Objectives}
\label{sec:training_objectives}

Offline demonstration initialization trains the value expert with the progress objective $\mathcal{L}_{\mathrm{init}}$ in Eq.~\ref{eq:progress_initialization} and the action expert with the behavior-cloning objective $\mathcal{L}_{\mathrm{BC}}$ in Eq.~\ref{eq:flow_matching_objectives}.
From offline RL warmup onward, the value expert uses the TD objective in Eq.~\ref{eq:td_value_learning}, and the action expert uses the advantage-conditioned objective $\mathcal{L}_{\mathrm{FM}}(\theta;o_t)$ in Eq.~\ref{eq:flow_matching_objectives}.
Algorithm~\ref{alg:training_pipeline} provides pseudocode for the complete offline-to-online RL system and denotes these operations by $\operatorname{Init}$, $\operatorname{BC}$, $\operatorname{TD}$, and $\operatorname{FM}$.

\subsubsection{Training Stages}

\paragraph{Offline demonstration initialization.}
Offline expert demonstrations $\mathcal{D}_{\mathrm{demo}}$ initialize $V_{\psi^{(0)}}$ and $\pi_{\theta^{(0)}}$ using the initialization objectives in Section~\ref{sec:training_objectives}.

\paragraph{Offline RL warmup.}
The warmup adapts both experts to deployment-specific states before online RL begins.
We deploy $(\pi_{\theta^{(0)}},V_{\psi^{(0)}})$ to collect policy rollouts and human-intervention (HI) trajectories for training the value expert.
Following the failed-prefix augmentation of GR-RL \citep{li2025grrl}, we construct prefixes from the rollout start to each human takeover.
Each prefix ends with zero reward and $d_t=1$, providing failure supervision for stuck or retry stages.
The complete rollout retains its outcome reward and bootstraps across takeovers to preserve recovery information.
\review{Thanks to the distributional value model, we can represent a mixture of the prefix and complete-rollout TD targets for shared samples.}

We temporally downsample the value-training trajectories to traverse them more times within a fixed update budget.
We sample from $\mathcal{D}_{\mathrm{demo}}$ and the processed warmup rollout data to update the value expert and obtain $V_{\psi_{\mathrm{warm}}}$.

We then annotate warmup trajectories with advantages from $V_{\psi_{\mathrm{warm}}}$ and train the action expert on $\mathcal{D}_{\mathrm{demo}}$ and these annotated trajectories to obtain $\theta_{\mathrm{warm}}$.

\paragraph{Online DAgger.}
Online DAgger co-adapts both experts to the current deployment distribution while human corrections continue to supervise recovery behavior.
With $\theta^{(k)}$ denoting round-$k$ policy parameters, we initialize the online rounds from $(\pi_{\theta^{(0)}},V_{\psi^{(0)}})=(\pi_{\theta_{\mathrm{warm}}},V_{\psi_{\mathrm{warm}}})$.
At each round $k$, we deploy $\pi_{\theta^{(k-1)}}$ with human intervention enabled and annotate the current rollouts with $V_{\psi^{(k-1)}}$.
We train the value expert on current rollouts and replay data comprising $\mathcal{D}_{\mathrm{demo}}$, warmup data, and earlier online rollouts.
We train the action expert on $\mathcal{D}_{\mathrm{demo}}$, historical human-controlled segments, and the annotated current rollouts.
The value and action experts are then updated synchronously, and the current trajectories are incorporated into the historical buffers for the next round.

\paragraph{Learn from Experience.}
LfE continues online improvement using sparse terminal rewards from autonomous rollouts, without collecting new human corrections.
We downsample the complete rollouts for value training and keep the historical intervention buffer fixed to the interventions collected before LfE.
Data mixing, annotation with $V_{\psi^{(k-1)}}$, and synchronous expert updates follow the Online DAgger procedure.

\subsection{Implementation Details}

The action expert inherits the complete architecture and pretrained parameters of $\pi_{0.5}$, including its PaliGemma-2B visual-language backbone and separate 300M-parameter action-generation transformer \citep{black2025pi05}.
The value expert repurposes the original $\pi_{0.5}$ 300M-parameter action-generation transformer as its visual-language backbone.
Following Multi-Scale Embodied Memory (MEM) \citep{torne2026mem}, the action and value experts use temporal-spatial attention to encode their respective image histories $h_t^\pi$ and $h_t^V$, each ending at observation timestep $t$.
\review{The action expert encodes past value predictions and proprioceptive states with a gated recurrent unit (GRU) and appends its final hidden state after the vision tokens.}
To predict categorical task progress, we append a learnable readout token to the end of the value backbone's input sequence after temporal-spatial history processing.
A linear head maps the readout token's output representation to $K$ logits, and a softmax produces the categorical value distribution over the fixed support in Sec.~\ref{sec:preliminaries}.
\review{Specific values of key hyperparameters are provided in Appendix~\ref{app:hyperparameters}.}

\section{Experiments}

\setlength{\parskip}{2pt}
\label{sec:experiments}

Our experiments aim to answer the following questions:
\begin{enumerate}
    \renewcommand{\labelenumi}{RQ\,\arabic{enumi}.}
    \item How much does value caching improve policy performance in RL post-training?
    \label{q:value_caching_performance}
    \item Why does value caching improve policy performance?
    \label{q:value_caching_mechanism}
    \item How much does VaM's offline-to-online RL system enhance sample efficiency?
    \label{q:offline_to_online_efficiency}
\end{enumerate}

\subsection{Tasks and Evaluation Protocol}
\label{subsec:experimental_setup}

We consider three real-world manipulation tasks, illustrated in Fig.~\ref{fig:task_illustrations}: \textit{Box Assembly}, \textit{Laundry Folding}, and \textit{Magnetic Bus Assembly}.
All three tasks are long-horizon manipulation workflows that cover contact-rich assembly and deformable-object manipulation.
Appendix~\ref{app:task_procedures} provides the detailed procedure and success criterion for each task.

\begin{figure}[!htbp]
    \setlength{\intextsep}{4pt}
    \setlength{\abovecaptionskip}{3pt}
    \setlength{\belowcaptionskip}{1pt}
    \centering
    \includegraphics[width=\textwidth]{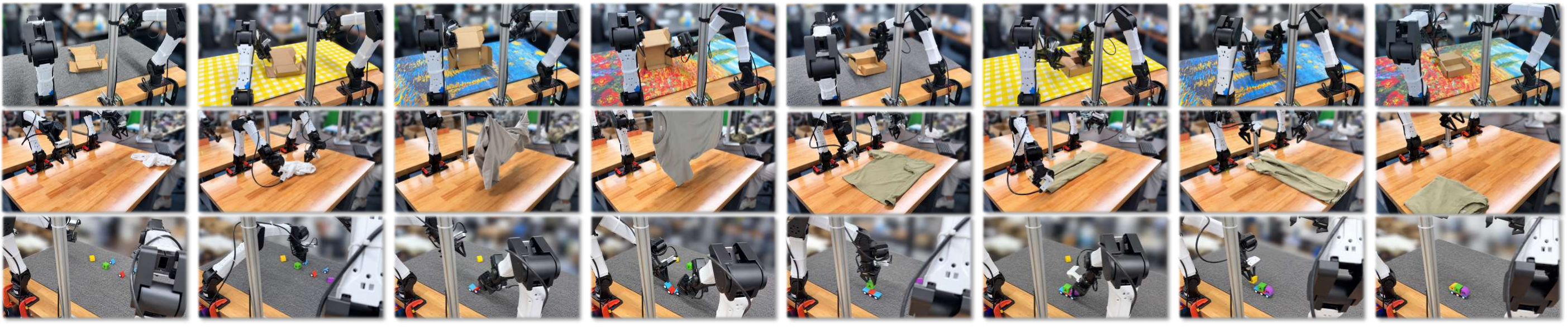}
    \caption{\textbf{Task illustrations for \textit{Box Assembly} (top), \textit{Laundry Folding} (middle), and \textit{Magnetic Bus Assembly} (bottom).}}
    \label{fig:task_illustrations}
\end{figure}

\paragraph{Evaluation protocol.}
We report success rate (SR), defined as the fraction of evaluation episodes completed successfully within the task-specific time limit.
For the \textit{Offline RL} and \textit{Online RL} comparison, we additionally report experimental wall-clock time per gradient step.
The real-robot setup and task-specific data-collection budgets are detailed in Appendix~\ref{app:experiment_setup} and \mbox{Table~\ref{tab:task_data_budgets}}.
\review{Detailed hyperparameters for the evaluation protocol are provided in Appendix~\ref{app:hyperparameters}.}

\Needspace{4\baselineskip}
\subsection{How Much Does Value Caching Improve Policy Performance in RL Post-Training?}
\label{subsec:value_caching_performance}

To answer RQ~\ref{q:value_caching_performance}, Fig.~\ref{fig:online_rl_performance} compares VaM with ablations that remove value caching or the offline RL warmup, and with RECAP \citep{physicalintelligence2025pi06}.
VaM and its ablation without value caching start from policies warmed up using their respective rollout and human-intervention trajectories.
\review{VaM without offline RL warmup and RECAP start from policies trained on $\mathcal{D}_{\mathrm{demo}}$ and receive additional Online DAgger rounds to compensate for the other methods' warmup data.}

\begin{figure}[!htbp]
    \setlength{\intextsep}{4pt}
    \setlength{\abovecaptionskip}{3pt}
    \setlength{\belowcaptionskip}{1pt}
    \centering
    \includegraphics[width=\textwidth]{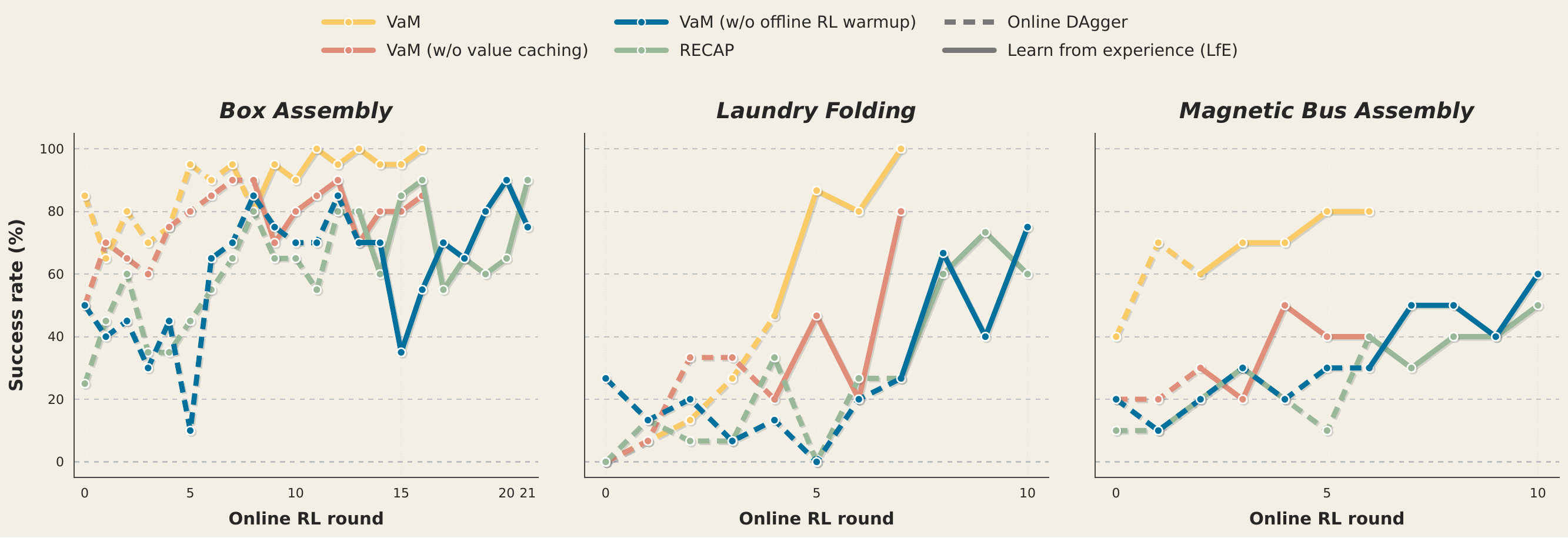}
    \caption{\textbf{Across \textit{Box Assembly}, \textit{Laundry Folding}, and \textit{Magnetic Bus Assembly}, value caching yields a mean absolute success-rate gain of $29.4\%$.}}
    \label{fig:online_rl_performance}
\end{figure}

Across all three tasks, VaM exceeds the ablation without value caching by $29.4\%$ in mean success rate over the final three checkpoints (Table~\ref{tab:online_rl_final_three}).

\Needspace{15\baselineskip}
\begin{wraptable}[15]{r}{0.58\textwidth}
    \vspace{-1pt}
    \centering
    \caption{\textbf{Success rate over the final three online RL checkpoints.}
    Entries are mean $\pm$ sample standard deviation across the final three online RL checkpoints.
    \review{By replacing on-policy value estimates with counterfactual estimates at inference, the bottom block demonstrates the importance of on-policy values for effective value caching.}}
    \label{tab:online_rl_final_three}
    \scriptsize
    \setlength{\tabcolsep}{2.2pt}
    \resizebox{\linewidth}{!}{%
    \begin{tabular}{@{}lccc@{}}
        \toprule
        \review{\textbf{Method}} & \review{\shortstack{\textbf{\textit{Box}}\\\textbf{\textit{Assembly}} $(\%)$}} & \review{\shortstack{\textbf{\textit{Laundry}}\\\textbf{\textit{Folding}} $(\%)$}} & \review{\shortstack{\textbf{\textit{Magnetic Bus}}\\\textbf{\textit{Assembly}} $(\%)$}} \\
        \midrule
        \textbf{VaM} & $\mathbf{96.7}\,{\color{gray}\pm 2.9}$ & $\mathbf{88.9}\,{\color{gray}\pm 10.2}$ & $\mathbf{76.7}\,{\color{gray}\pm 5.8}$ \\
        \textbf{VaM (w/o value caching)} & $81.7\,{\color{gray}\pm 2.9}$ & $48.9\,{\color{gray}\pm 30.1}$ & $43.3\,{\color{gray}\pm 5.8}$ \\
        \textbf{VaM (w/o offline RL warmup)} & $81.7\,{\color{gray}\pm 7.6}$ & $60.6\,{\color{gray}\pm 18.3}$ & $50.0\,{\color{gray}\pm 10.0}$ \\
        \textbf{RECAP} & $71.7\,{\color{gray}\pm 16.1}$ & $64.4\,{\color{gray}\pm 7.7}$ & $43.3\,{\color{gray}\pm 5.8}$ \\
        \midrule
        \reviewtablerows{exp_fig/analysis/review-evidence/outputs/replacement-main-rows.tex}
        \bottomrule
    \end{tabular}%
    }
\end{wraptable}

The gap is most pronounced during LfE and persists after the policy stops receiving on-policy DAgger data.
\review{Averaging all available LfE checkpoints yields a consistent gap of $31.3\%$ (full results in Appendix~\ref{app:checkpoint_sensitivity}).}
\review{These observations support the use of cached task-progress estimates during LfE.}
\review{VaM also exceeds its ablation without offline RL warmup on all three tasks, supporting the benefit of offline RL warmup for final performance.}

\paragraph{\review{Importance of on-policy value estimates.}}
\review{To probe the importance of accurate, on-policy estimates for value caching, we retain the final three LfE action experts and replace only their value experts with snapshots obtained after progress initialization on $\mathcal{D}_{\mathrm{demo}}$ or offline RL warmup.}
\review{These substitutions provide counterfactual value estimates while preserving the value caching mechanism across all three tasks.}
\review{Both earlier value snapshots yield lower success rates than the original LfE value experts on all three tasks (Table~\ref{tab:online_rl_final_three}), supporting the importance of value estimates adapted to the current policy for effective value caching.}
\review{Appendix~\ref{app:value_replacement} details the replacement protocol.}

\subsection{Why Does Value Caching Improve Policy Performance?}
\label{subsec:value_caching_mechanism}

\review{To answer RQ~\ref{q:value_caching_mechanism}, Fig.~\ref{fig:value_caching_mechanism} compares principal component analysis (PCA) projections of value-expert embeddings, value predictions, and robot behavior on \textit{Box Assembly}.}
\review{In the displayed cached rollouts, PCA projection trajectories are smoother and ordered by progress, with deviations at stuck states preceding corrective actions.}
\review{Without caching, stages are more intermingled, with low or fluctuating values during failed execution.}
\review{Appendix~\ref{app:value_caching_qualitative} extends these visualizations to \textit{Laundry Folding} and \textit{Magnetic Bus Assembly} and details a \textit{Laundry Folding} comparison probing partial observability: we hold observations fixed and compare policy predictions using demo-trained and post-warmup value replacements following Sec.~\ref{subsec:value_caching_performance}.}

\begin{figure}[t]
    \centering
    \setlength{\parskip}{0pt}
    \includegraphics[width=0.47\textwidth]{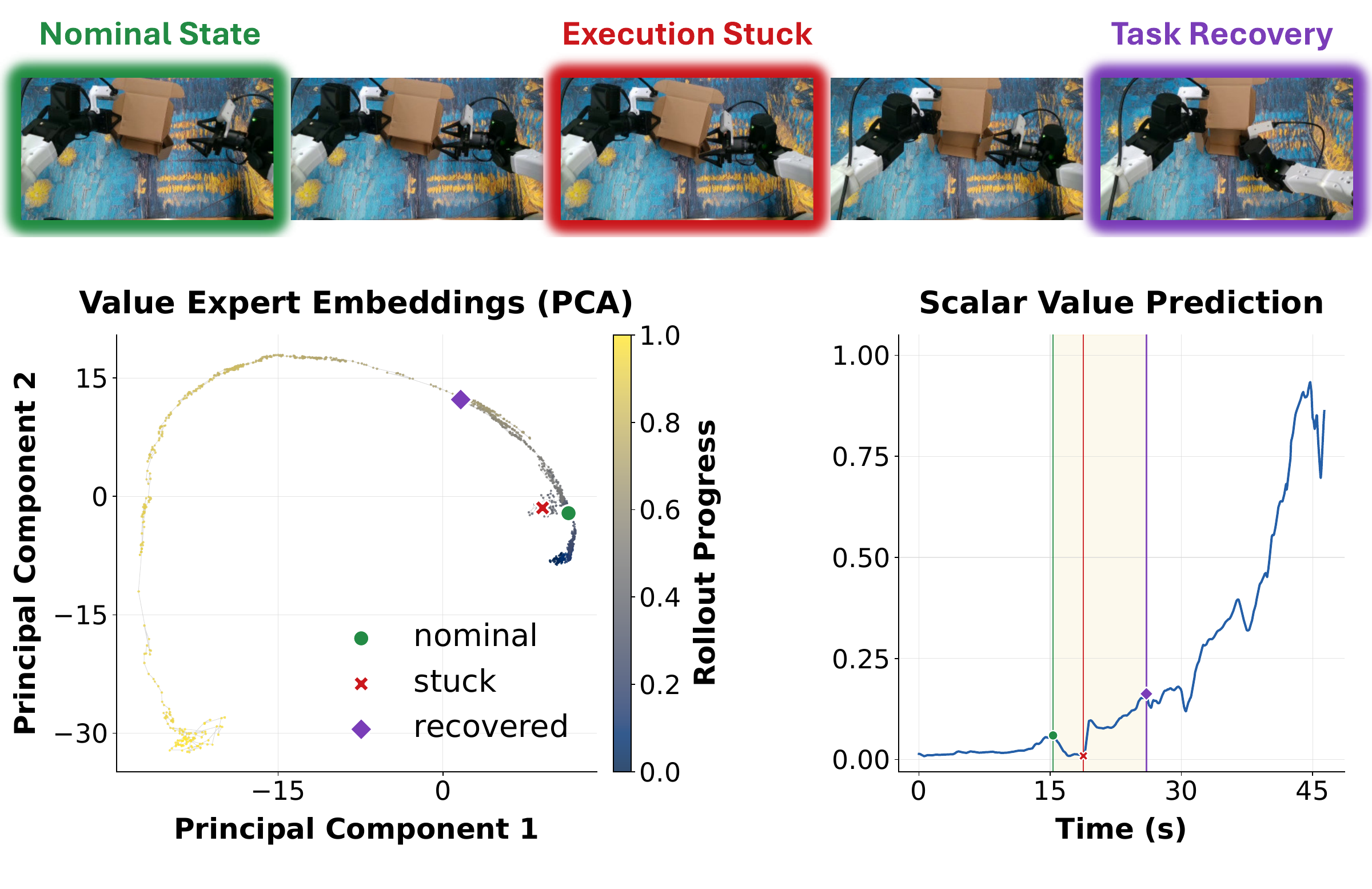}%
    \hspace{0.01\textwidth}%
    \includegraphics[width=0.47\textwidth]{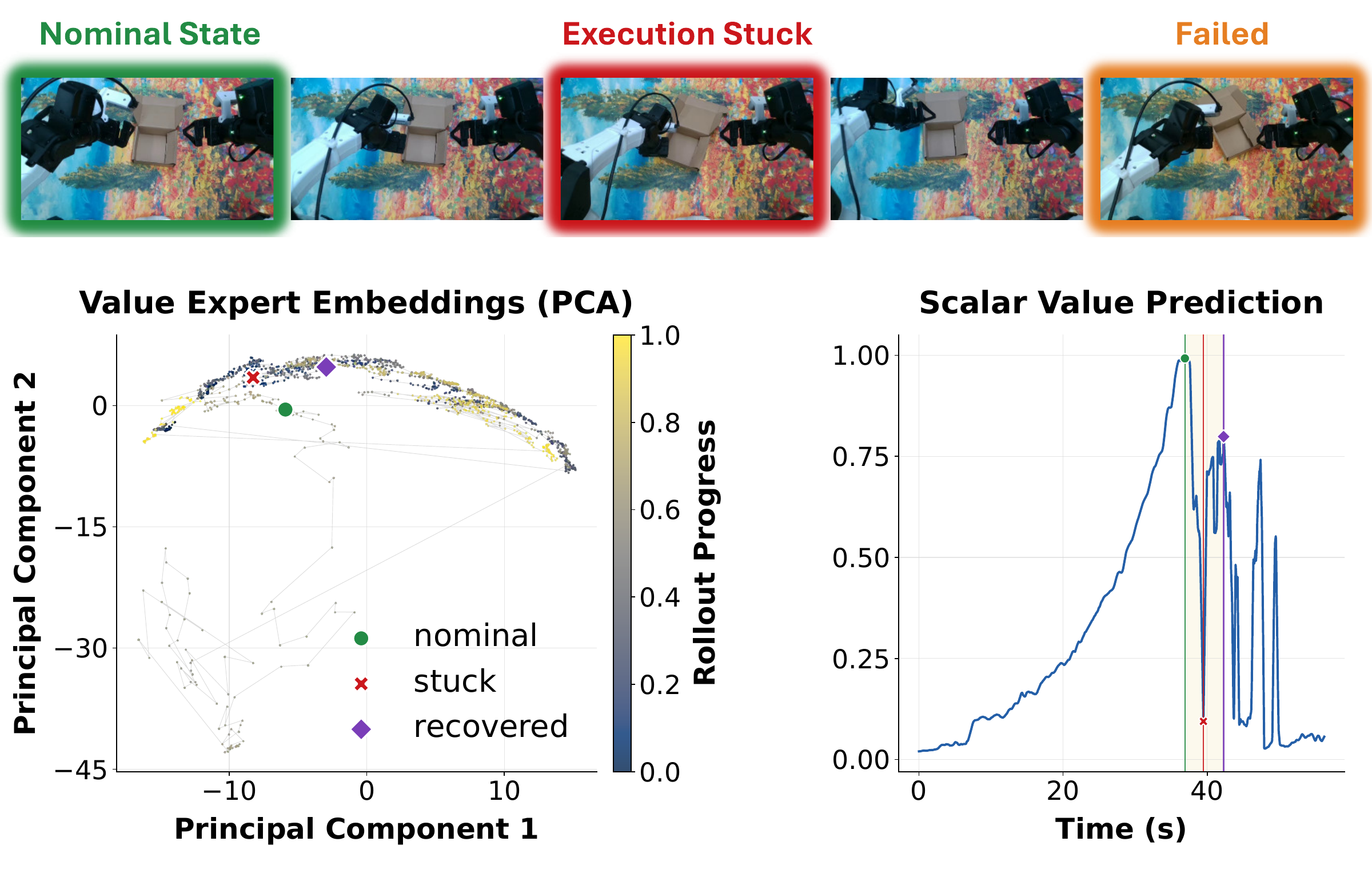}
    \par\vspace{1pt}\nointerlineskip
    \includegraphics[width=0.47\textwidth]{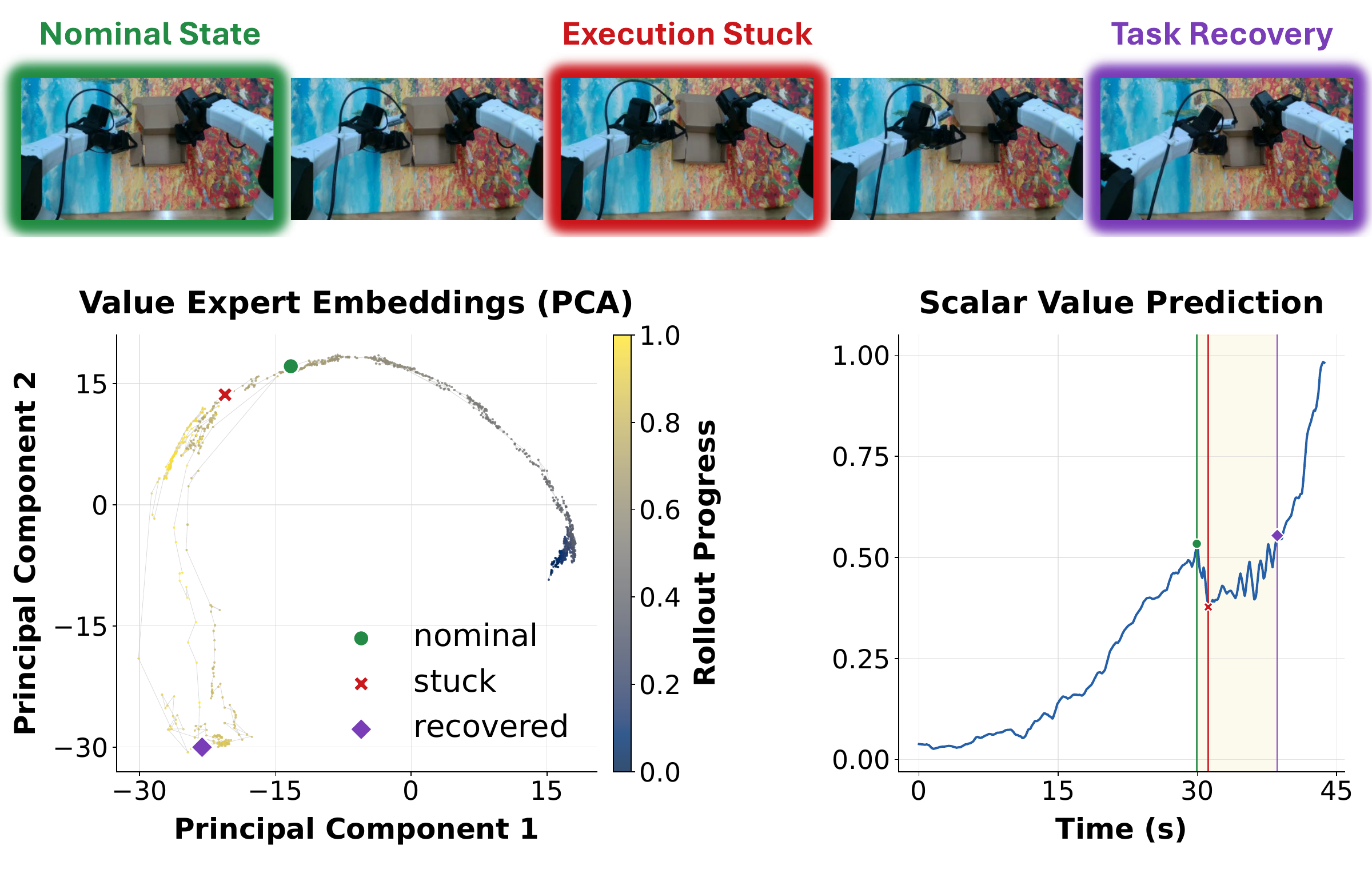}%
    \hspace{0.01\textwidth}%
    \includegraphics[width=0.47\textwidth]{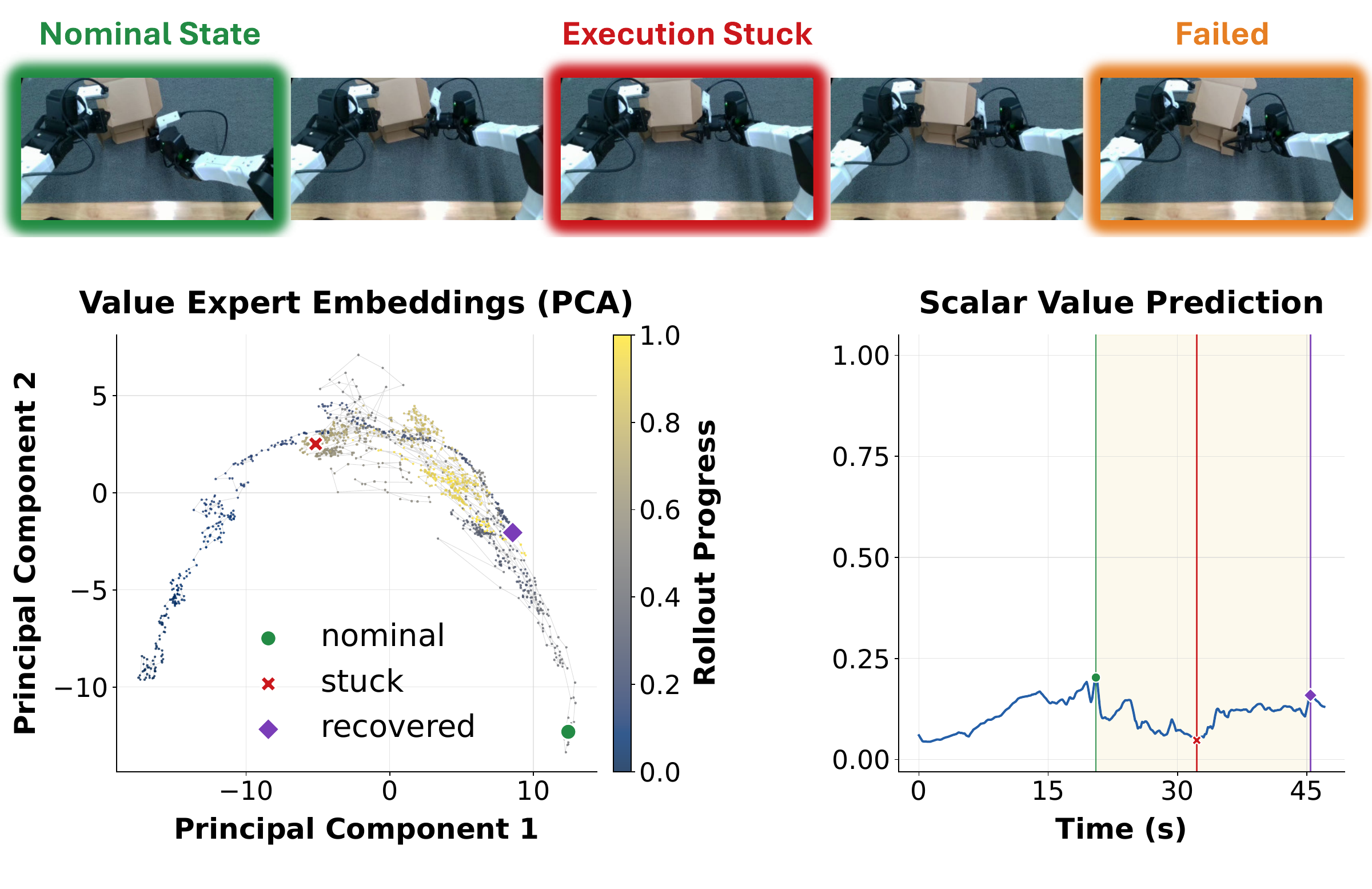}
    \par\vspace{1pt}\nointerlineskip
    \includegraphics[width=0.47\textwidth]{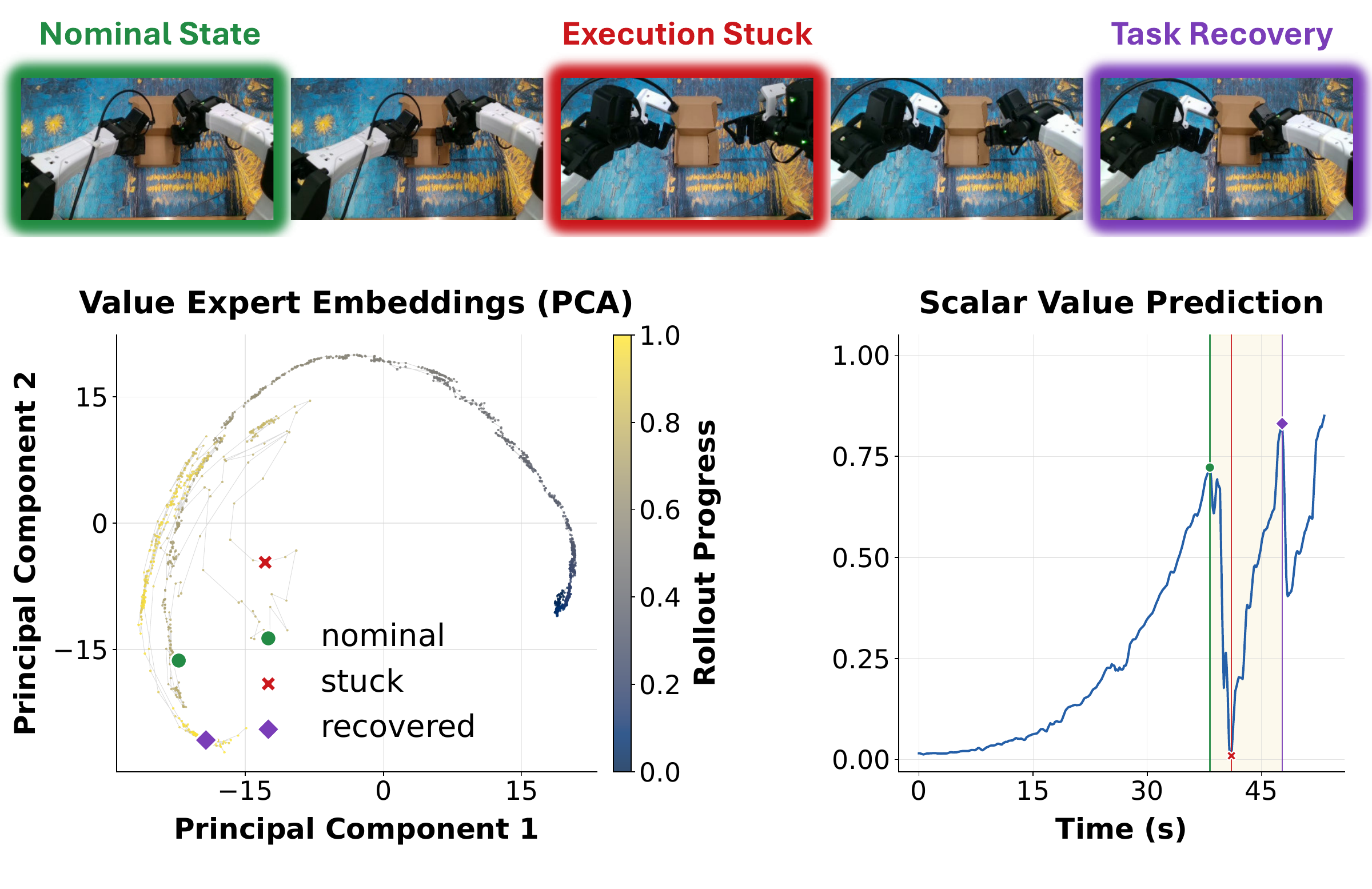}%
    \hspace{0.01\textwidth}%
    \includegraphics[width=0.47\textwidth]{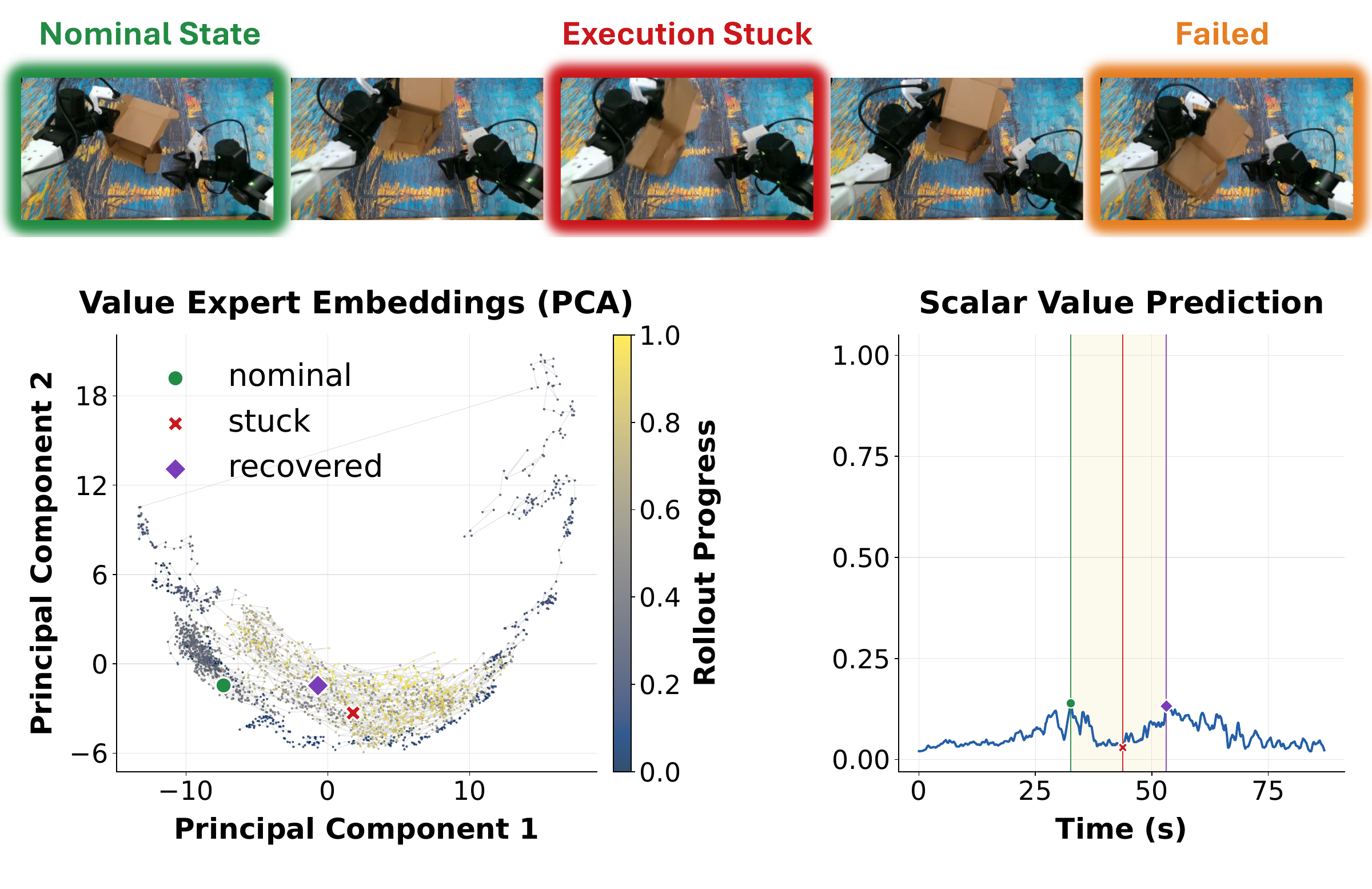}
    \caption{\textbf{Value caching supports progress-aware recovery: VaM (left) versus VaM (w/o value caching) (right) on \textit{Box Assembly}.}
    We illustrate the eye-to-hand camera images, PCA embeddings of the value expert colored by rollout progress, and scalar value predictions.}
    \label{fig:value_caching_mechanism}
\end{figure}

\paragraph{Recovery behavior during LfE.}
\review{To assess corrective behavior, we conduct blind testing on all LfE rollouts from the four methods, where human annotators do not know which method produced each rollout.}

\begin{wraptable}[\ifdefined\arxivversion 6\else 8\fi]{r}{0.52\textwidth}
    \vspace{-14pt}
    \centering
    \caption{\textbf{Recovery rate during LfE.}}
    \label{tab:recovery_rate}
    \scriptsize
    \setlength{\tabcolsep}{2.5pt}
    \resizebox{\linewidth}{!}{%
    \begin{tabular}{@{}lccc@{}}
        \toprule
        \review{\textbf{Method}} & \review{\shortstack{\textbf{\textit{Box}}\\\textbf{\textit{Assembly}} $(\%)$}} & \review{\shortstack{\textbf{\textit{Laundry}}\\\textbf{\textit{Folding}} $(\%)$}} & \review{\shortstack{\textbf{\textit{Magnetic Bus}}\\\textbf{\textit{Assembly}} $(\%)$}} \\
        \midrule
        \textbf{VaM} & $\mathbf{54.5}$ & $\mathbf{86.4}$ & $\mathbf{33.3}$ \\
        \textbf{VaM (w/o value caching)} & $50.0$ & $58.8$ & $7.4$ \\
        \textbf{VaM (w/o offline RL warmup)} & $25.3$ & $72.7$ & \review{$27.3$} \\
        \textbf{RECAP} & $42.4$ & $76.6$ & \review{$13.3$} \\
        \bottomrule
    \end{tabular}%
    }
\end{wraptable}

They count events requiring recovery and successful recoveries in which the policy executes corrective actions to complete the affected stage.
We define \textit{Recovery Rate} as the ratio of successful recoveries to recovery-requiring events.

\review{VaM exceeds its caching ablation in observed recovery rate on all three tasks (Table~\ref{tab:recovery_rate}), with detailed experimental settings and event counts in Appendix~\ref{app:recovery_counts}.}
\review{The qualitative separation is consistent with the conditional soft-aliasing analysis in Sec.~\ref{sec:soft_aliasing_analysis}.}

\subsection{How Much Does \textnormal{VaM}'s Offline-to-Online RL System Enhance Sample Efficiency?}
\label{subsec:offline_to_online_efficiency}

\review{To answer RQ~\ref{q:offline_to_online_efficiency}, we compare \textit{Offline RL} and \textit{Online RL}, initialized from the same VaM model after offline RL warmup, under matched gradient-step and data-collection budgets.}
\review{Figures~\ref{fig:offline_online_success} and~\ref{fig:offline_online_gradient_time} report success rate over the final three checkpoints of each stage and experimental time per gradient step.}

\begin{figure}[!htbp]
    \setlength{\intextsep}{4pt}
    \setlength{\abovecaptionskip}{3pt}
    \setlength{\belowcaptionskip}{1pt}
    \centering
    \begin{minipage}[t]{0.49\textwidth}
        \centering
        \includegraphics[height=0.524\linewidth]{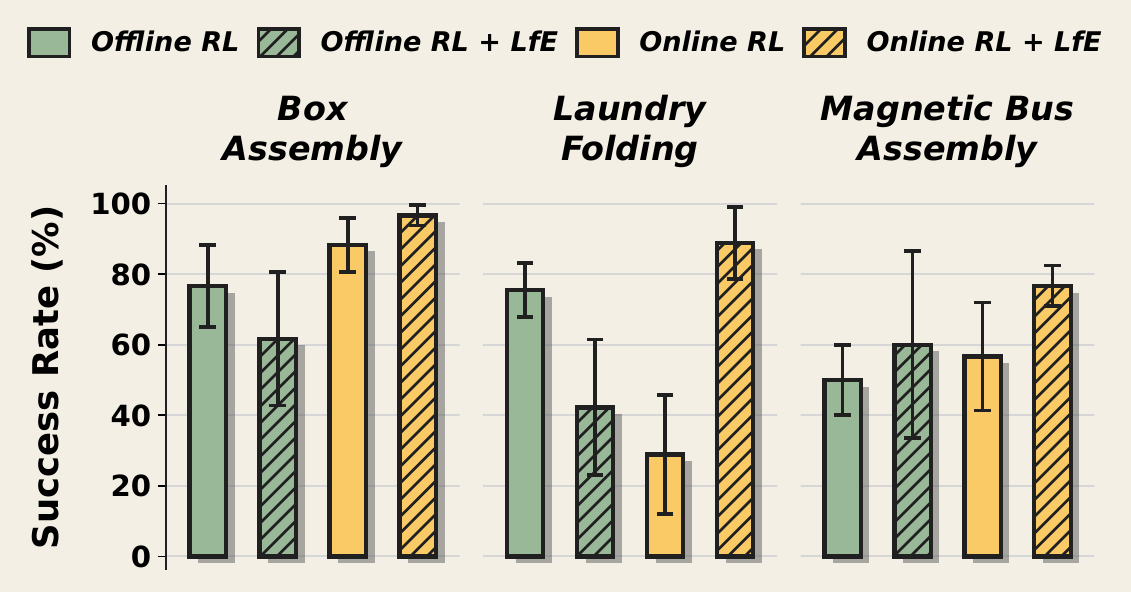}
        \caption{\textbf{Across all tasks, VaM yields an average absolute success-rate gain of $32.8\%$ over matched \textit{Offline RL} baselines after LfE.}
        \review{Bars show mean $\pm$ sample standard deviation.}}
        \label{fig:offline_online_success}
    \end{minipage}
    \hfill
    \begin{minipage}[t]{0.49\textwidth}
        \centering
        \includegraphics[height=0.524\linewidth]{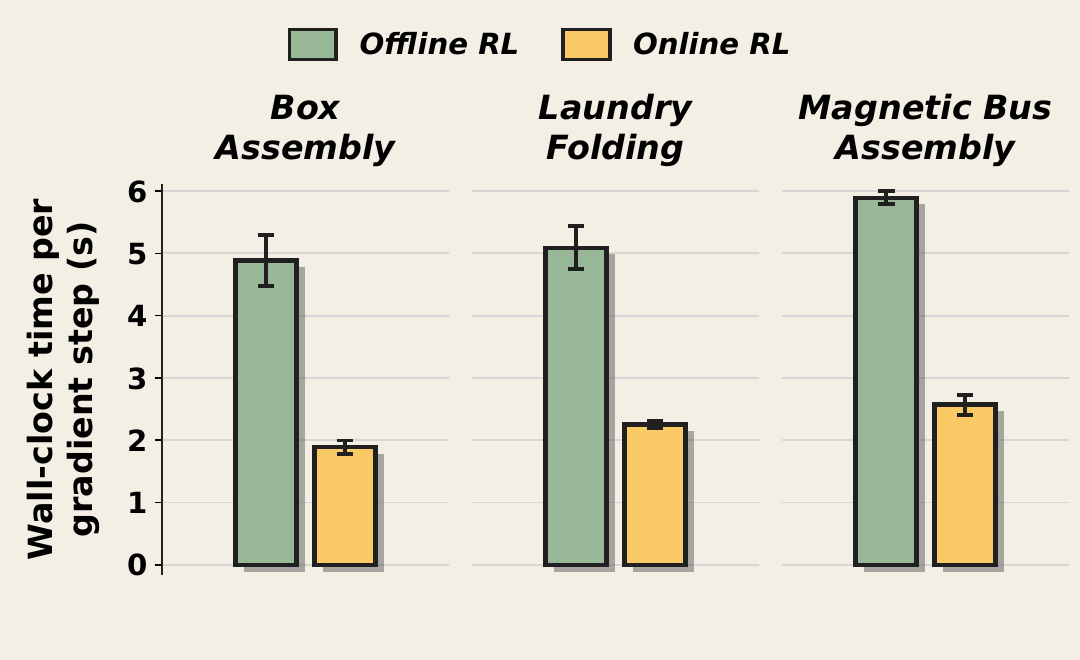}
        \caption{\review{\textbf{Across all tasks, \textit{Online RL} reduces experimental time per gradient step by $57.8\%$ on average.}}
        \review{Bars show mean $\pm$ sample standard deviation.}}
        \label{fig:offline_online_gradient_time}
    \end{minipage}
\end{figure}

\review{After LfE, VaM exceeds \textit{Offline RL} on all three tasks by $32.8\%$ on average, supporting online closed-loop policy updates under the evaluated budgets.}
\review{VaM reduces mean experimental time per gradient step by $57.8\%$ across tasks by running data collection, value training, and action training concurrently, while \textit{Offline RL} performs them sequentially.}
\review{For example, on \textit{Box Assembly}, \textit{Online RL} uses 16 iterations versus 2 for \textit{Offline RL} under matched gradient-step budgets (detailed experimental settings in Appendix~\ref{app:time_accounting}).}

\setlength{\parskip}{3pt}

\section{Related Work}

\subsection{Post-training of Generalist Robot Policies}

Policy-centric post-training adapts generalist robot policies through residual updates or RL objectives tailored to consistency, diffusion, and flow-based policies \citep{xiao2026pld,chen2025conrft,li2026simplevlarl,chen2025pirl}.
Likelihood-based, hierarchical, and latent-space formulations support optimization of generative policies \citep{ren2025diffusion,lei2026rl100,wagenmaker2025dsrl,li2025grrl,chen2025pirl}, while world models reduce physical interaction by providing imagined or synthetic experience \citep{yang2026rise,zhu2026wmpo,gigabrainteam2026gigabrain05m}.
To learn from a robot's own experience, some prior methods learn an additional value function to filter out suboptimal segments.
General-purpose reward and critic models learn transferable evaluations from robot trajectories or vision-language supervision \citep{liang2026robometer,tan2026robodopamine,zhai2026vlac}, while advantage-oriented methods derive relative progress from temporal relations in demonstrations and trajectories \citep{liu2026steam,yu2026chi0,zhang2025rewind}.
RECAP combines demonstrations, policy rollouts, and human interventions to estimate value and annotate data for advantage-conditioned policy extraction through CFGRL \citep{physicalintelligence2025pi06,frans2025diffusion}.
However, learning value estimates from fixed or periodically refreshed datasets can leave the critic stale as the deployment distribution changes between updates.
In addition, when value predictions serve only as rewards, filtering scores, annotations, or optimization weights, the deployed policy does not receive them as inputs.
Without this explicit task-progress information, visually similar observations can obscure different stages of execution, leaving the policy with partial observability of task progress.
VaM instead continually updates value from online rollouts and human interventions and uses its predictions both as policy-learning supervision and an explicit policy input.

\subsection{Human-in-the-loop Online Adaptation}

Human operators intervene during deployment to correct uncertain or failing policy behavior \citep{luo2025humanagent,deng2026e2hil,hu2026rac}.
Methods such as HG-DAgger and Sirius use interventions as corrective demonstrations for policy improvement \citep{kelly2019hgdagger,liu2023sirius}.
Fleet-DAgger, SiriusFleet, and ARMADA scale intervention collection by allocating operators or requesting assistance only when necessary \citep{hoque2023fleetdagger,liu2025siriusfleet,yu2026armada}, while recent fleet-scale systems combine online robot rollouts with human corrections across multiple robots and tasks \citep{pan2026sop,wang2026learningwhiledeploying}.
VaM further uses intervention trajectories to update a value model that guides policy learning and conditions the policy input.

\subsection{Offline-to-Online Reinforcement Learning}

Offline-to-online RL stabilizes policy improvement by retaining or replaying offline data during online training \citep{ball2023rlpd,peng2019advantage,nair2020awac,li2025grrl}.
RLPD mixes offline and online samples, Cal-QL calibrates an offline value function for online fine-tuning, and Warm-Start RL initializes online learning with replay data that preserves pretrained behavior \citep{ball2023rlpd,nakamoto2023calql,zhou2025efficient}.
In real-world robotics, SERL and HIL-SERL combine behavior-cloning initialization with off-policy learning for physical manipulation \citep{luo2024serl,luo2025hilserl}.
VaM combines offline demonstrations, historical off-policy intervention data, and current on-policy experience for sample-efficient online adaptation.

\section{Discussions and Conclusions}

Though VaM demonstrates improvement over prior methods in sample efficiency and mitigates partial observability in generalist robot policies, it has several limitations that suggest directions for future work.
First, value caching depends on a value function that closely tracks actual task progress, making offline RL warmup necessary for stable online training.
The cached history consists of the value expert's own predictions, so its usefulness as task-progress memory depends on the quality of those predictions.
Second, our real-world RL framework is currently limited to a single robot.
Our experiments do not establish how the framework scales when multiple robots contribute data.
Scaling VaM to robot fleets with shared experience is a promising direction for further improving sample efficiency.

In conclusion, we presented VaM, a framework that uses recent value predictions as task-progress memory and combines offline demonstrations, off-policy trajectories, and on-policy rollouts for online policy and value adaptation.
Value caching addresses action ambiguity under partial observability, while the offline-to-online RL pipeline supports continued learning from deployment experience.
Together, these components offer a practical approach to post-training generalist robot policies.

\clearpage
\setlength{\parskip}{6pt}
\setlength{\textfloatsep}{8pt}
\setlength{\abovedisplayskip}{7pt plus 2pt minus 5pt}
\setlength{\belowdisplayskip}{7pt plus 2pt minus 5pt}
\flushbottom
\subsection*{AI use statement}

We used generative artificial intelligence (AI) tools to polish writing, assist with research execution, and draft sections of the paper.
We did not use these tools for literature retrieval or discovery, research ideation, synthetic dataset generation, or proving mathematical claims.
The authors take responsibility for the final content of this paper.

\subsection*{Reproducibility statement}

All code will be released as open source and is currently being organized for release.

%
%

%

\bibliography{iclr2027_conference}

@inproceedings{chen2025conrft,
  title     = {{ConRFT}: A Reinforced Fine-tuning Method for {VLA} Models via Consistency Policy},
  author    = {Chen, Yuhui and Tian, Shuai and Liu, Shugao and Zhou, Yingting and Li, Haoran and Zhao, Dongbin},
  booktitle = {Proceedings of Robotics: Science and Systems},
  year      = {2025},
  address   = {Los Angeles, CA, USA},
  month     = jun,
  doi       = {10.15607/RSS.2025.XXI.019}
}

@inproceedings{li2026simplevlarl,
  title     = {{SimpleVLA-RL}: Scaling {VLA} Training via Reinforcement Learning},
  author    = {Li, Haozhan and Zuo, Yuxin and Yu, Jiale and Zhang, Yuhao and Yang, Zhaohui and Zhang, Kaiyan and Zhu, Xuekai and Zhang, Yuchen and Chen, Tianxing and Cui, Ganqu and Wang, Dehui and Luo, Dingxiang and Fan, Yuchen and Sun, Youbang and Zeng, Jia and Pang, Jiangmiao and Zhang, Shanghang and Wang, Yu and Mu, Yao and Zhou, Bowen and Ding, Ning},
  booktitle = {International Conference on Learning Representations},
  year      = {2026},
  eprint    = {2509.09674},
  archivePrefix = {arXiv},
  primaryClass  = {cs.RO}
}

@inproceedings{xiao2026pld,
  title     = {Self-Improving Vision-Language-Action Models with Data Generation via Residual {RL}},
  author    = {Xiao, Wenli and Lin, Haotian and Peng, Andy and Xue, Haoru and He, Tairan and Xie, Yuqi and Hu, Fengyuan and Wu, Jimmy and Luo, Zhengyi and Fan, Linxi Jim and Shi, Guanya and Zhu, Yuke},
  booktitle = {International Conference on Learning Representations},
  year      = {2026},
  eprint    = {2511.00091},
  archivePrefix = {arXiv},
  primaryClass  = {cs.RO}
}

@article{chen2025pirl,
  title   = {{$\pi_{\mathrm{RL}}$}: Online {RL} Fine-tuning for Flow-based Vision-Language-Action Models},
  author  = {Chen, Kang and Liu, Zhihao and Zhang, Tonghe and Guo, Zhen and Xu, Si and Lin, Hao and Zang, Hongzhi and Li, Xiang and Zhang, Quanlu and Yu, Zhaofei and Fan, Guoliang and Huang, Tiejun and Wang, Yu and Yu, Chao},
  journal = {arXiv preprint arXiv:2510.25889},
  year    = {2025},
  eprint  = {2510.25889},
  archivePrefix = {arXiv},
  primaryClass  = {cs.RO}
}

@inproceedings{zhai2026vlac,
  title     = {A Vision-Language-Action-Critic Model for Robotic Real-World Reinforcement Learning},
  author    = {Zhai, Shaopeng and Zhang, Qi and Zhang, Tianyi and Huang, Fuxian and Zhang, Haoran and Zhou, Ming and Zhang, Shengzhe and Liu, Litao and Lin, Sixu and Pang, Jiangmiao},
  booktitle = {International Conference on Machine Learning},
  year      = {2026},
  eprint    = {2509.15937},
  archivePrefix = {arXiv},
  primaryClass  = {cs.RO}
}

@article{liu2026steam,
  title   = {{STEAM}: Self-Supervised Temporal Ensemble Advantage Modeling for Real-World Robot Learning},
  author  = {Liu, Zhihao and Gu, Qiuyi and Wang, Yitao and Qiao, Dongming and Zhang, Yixian and Chen, Shuaihang and Shi, Liangzhi and Zhou, Tianxing and Huang, Zefang and Chen, Kang and Guo, Zhen and Zhang, Quanlu and Yu, Jincheng and Liang, Xiaodan and Fan, Guoliang and Wang, Yu and Gao, Feng and Chen, Xinlei and Yu, Chao},
  journal = {arXiv preprint arXiv:2606.29834},
  year    = {2026},
  eprint  = {2606.29834},
  archivePrefix = {arXiv},
  primaryClass  = {cs.RO}
}

@article{physicalintelligence2025pi06,
  title   = {{$\pi^{*}_{0.6}$}: a {VLA} That Learns From Experience},
  author  = {{Physical Intelligence} and Amin, Ali and Aniceto, Raichelle and Balakrishna, Ashwin and Black, Kevin and Conley, Ken and Connors, Grace and Darpinian, James and Dhabalia, Karan and DiCarlo, Jared and Driess, Danny and Equi, Michael and Esmail, Adnan and Fang, Yunhao and Finn, Chelsea and Glossop, Catherine and Godden, Thomas and Goryachev, Ivan and Groom, Lachy and Hancock, Hunter and Hausman, Karol and Hussein, Gashon and Ichter, Brian and Jakubczak, Szymon and Jen, Rowan and Jones, Tim and Katz, Ben and Ke, Liyiming and Kuchi, Chandra and Lamb, Marinda and LeBlanc, Devin and Levine, Sergey and Li-Bell, Adrian and Lu, Yao and Mano, Vishnu and Mothukuri, Mohith and Nair, Suraj and Pertsch, Karl and Ren, Allen Z. and Sharma, Charvi and Shi, Lucy Xiaoyang and Smith, Laura and Springenberg, Jost Tobias and Stachowicz, Kyle and Stoeckle, Will and Swerdlow, Alex and Tanner, James and Torne, Marcel and Vuong, Quan and Walling, Anna and Wang, Haohuan and Williams, Blake and Yoo, Sukwon and Yu, Lili and Zhilinsky, Ury and Zhou, Zhiyuan},
  journal = {arXiv preprint arXiv:2511.14759},
  year    = {2025},
  eprint  = {2511.14759},
  archivePrefix = {arXiv},
  primaryClass  = {cs.RO}
}

@article{yu2026chi0,
  title   = {{$\chi_0$}: Resource-Aware Robust Manipulation via Taming Distributional Inconsistencies},
  author  = {Yu, Checheng and Sima, Chonghao and Jiang, Gangcheng and Zhang, Hai and Mai, Haoguang and Li, Hongyang and Wang, Huijie and Chen, Jin and Wu, Kaiyang and Chen, Li and Zhao, Lirui and Shi, Modi and Luo, Ping and Bu, Qingwen and Peng, Shijia and Li, Tianyu and Yuan, Yibo},
  journal = {arXiv preprint arXiv:2602.09021},
  year    = {2026},
  eprint  = {2602.09021},
  archivePrefix = {arXiv},
  primaryClass  = {cs.RO}
}

@article{gigabrainteam2026gigabrain05m,
  title   = {{GigaBrain-0.5M*}: a {VLA} That Learns From World Model-Based Reinforcement Learning},
  author  = {{GigaBrain Team} and Wang, Boyuan and Li, Bohan and Ni, Chaojun and Huang, Guan and Zhao, Guosheng and Li, Hao and Li, Jie and Lv, Jindi and Liu, Jingyu and Feng, Lv and Yu, Mingming and Li, Peng and Deng, Qiuping and Liu, Tianze and Zhou, Xinyu and Chen, Xinze and Wang, Xiaofeng and Wang, Yang and Li, Yifan and Nie, Yifei and Li, Yilong and Zhou, Yukun and Ye, Yun and Liu, Zhichao and Zhu, Zheng},
  journal = {arXiv preprint arXiv:2602.12099},
  year    = {2026},
  eprint  = {2602.12099},
  archivePrefix = {arXiv},
  primaryClass  = {cs.RO}
}

@inproceedings{yang2026rise,
  title     = {{RISE}: Self-Improving Robot Policy with Compositional World Model},
  author    = {Yang, Jiazhi and Lin, Kunyang and Li, Jinwei and Zhang, Wencong and Lin, Tianwei and Wu, Longyan and Su, Zhizhong and Zhao, Hao and Zhang, Ya-Qin and Chen, Li and Luo, Ping and Yue, Xiangyu and Li, Hongyang},
  booktitle = {Proceedings of Robotics: Science and Systems},
  year      = {2026},
  eprint    = {2602.11075},
  archivePrefix = {arXiv},
  primaryClass  = {cs.RO}
}

@inproceedings{zhu2026wmpo,
  title     = {{WMPO}: World Model-based Policy Optimization for Vision-Language-Action Models},
  author    = {Zhu, Fangqi and Yan, Zhengyang and Hong, Zicong and Shou, Quanxin and Ma, Xiao and Guo, Song},
  booktitle = {International Conference on Learning Representations},
  year      = {2026},
  eprint    = {2511.09515},
  archivePrefix = {arXiv},
  primaryClass  = {cs.RO}
}

@article{hu2026rac,
  title   = {{RaC}: Robot Learning for Long-Horizon Tasks by Scaling Recovery and Correction},
  author  = {Hu, Zheyuan and Wu, Robyn and Enock, Naveen and Li, Jasmine and Kadakia, Riya and Erickson, Zackory and Kumar, Aviral},
  journal = {IEEE Transactions on Robotics},
  volume  = {42},
  pages   = {2843--2860},
  year    = {2026},
  eprint  = {2509.07953},
  archivePrefix = {arXiv},
  primaryClass  = {cs.RO}
}

@article{pan2026sop,
  title   = {{SOP}: A Scalable Online Post-Training System for Vision-Language-Action Models},
  author  = {Pan, Mingjie and Feng, Siyuan and Zhang, Qinglin and Li, Xinchen and Song, Jianheng and Qu, Chendi and Wang, Yi and Li, Chuankang and Xiong, Ziyu and Chen, Zhi and Liu, Yi and Luo, Jianlan},
  journal = {arXiv preprint arXiv:2601.03044},
  year    = {2026},
  eprint  = {2601.03044},
  archivePrefix = {arXiv},
  primaryClass  = {cs.RO}
}

@article{wang2026learningwhiledeploying,
  title   = {Learning while Deploying: Fleet-Scale Reinforcement Learning for Generalist Robot Policies},
  author  = {Wang, Yi and Li, Xinchen and Xie, Pengwei and Yang, Pu and Nie, Buqing and Cai, Yunuo and Zhang, Qinglin and Qu, Chendi and Wu, Jeffrey and Song, Jianheng and Ren, Xinlin and Huang, Jingshun and Pan, Mingjie and Feng, Siyuan and Chen, Zhi and Luo, Jianlan},
  journal = {arXiv preprint arXiv:2605.00416},
  year    = {2026},
  eprint  = {2605.00416},
  archivePrefix = {arXiv},
  primaryClass  = {cs.RO}
}

@inproceedings{hoque2023fleetdagger,
  title     = {{Fleet-DAgger}: Interactive Robot Fleet Learning with Scalable Human Supervision},
  author    = {Hoque, Ryan and Chen, Lawrence Yunliang and Sharma, Satvik and Dharmarajan, Karthik and Thananjeyan, Brijen and Abbeel, Pieter and Goldberg, Ken},
  booktitle = {Proceedings of The 6th Conference on Robot Learning},
  series    = {Proceedings of Machine Learning Research},
  volume    = {205},
  pages     = {368--380},
  year      = {2023},
  publisher = {PMLR}
}

@inproceedings{liu2023sirius,
  title     = {Robot Learning on the Job: Human-in-the-Loop Autonomy and Learning During Deployment},
  author    = {Liu, Huihan and Nasiriany, Soroush and Zhang, Lance and Bao, Zhiyao and Zhu, Yuke},
  booktitle = {Proceedings of Robotics: Science and Systems},
  year      = {2023},
  address   = {Daegu, Republic of Korea},
  month     = jul,
  doi       = {10.15607/RSS.2023.XIX.005}
}

@article{yu2026armada,
  title   = {{ARMADA}: Autonomous Online Failure Detection and Human Shared Control Empower Scalable Real-world Deployment and Adaptation},
  author  = {Yu, Wenye and Lv, Jun and Ying, Zixi and Jin, Yang and Wen, Chuan and Lu, Cewu},
  journal = {IEEE Robotics and Automation Letters},
  year    = {2026},
  doi     = {10.1109/LRA.2026.3709635}
}

@inproceedings{liu2025siriusfleet,
  title     = {Multi-Task Interactive Robot Fleet Learning with Visual World Models},
  author    = {Liu, Huihan and Zhang, Yu and Betala, Vaarij and Zhang, Evan and Liu, James and Ding, Crystal and Zhu, Yuke},
  booktitle = {Proceedings of The 8th Conference on Robot Learning},
  series    = {Proceedings of Machine Learning Research},
  volume    = {270},
  pages     = {4286--4313},
  year      = {2025},
  publisher = {PMLR}
}

@inproceedings{kelly2019hgdagger,
  title     = {{HG-DAgger}: Interactive Imitation Learning with Human Experts},
  author    = {Kelly, Michael and Sidrane, Chelsea and Driggs-Campbell, Katherine and Kochenderfer, Mykel J.},
  booktitle = {2019 International Conference on Robotics and Automation (ICRA)},
  pages     = {8077--8083},
  year      = {2019},
  publisher = {IEEE},
  doi       = {10.1109/ICRA.2019.8793698}
}

@article{lei2026rl100,
  title   = {Performant Robotic Manipulation with Real-World Reinforcement Learning},
  author  = {Lei, Kun and Li, Huanyu and Yu, Dongjie and Wei, Zhenyu and Guo, Lingxiao and Jiang, Zhennan and Wang, Ziyu and Liang, Shiyu and Xu, Huazhe},
  journal = {Science Robotics},
  volume  = {11},
  number  = {116},
  pages   = {eaed6267},
  year    = {2026},
  doi     = {10.1126/scirobotics.aed6267}
}

@inproceedings{luo2024serl,
  title     = {{SERL}: A Software Suite for Sample-Efficient Robotic Reinforcement Learning},
  author    = {Luo, Jianlan and Hu, Zheyuan and Xu, Charles and Tan, You Liang and Berg, Jacob and Sharma, Archit and Schaal, Stefan and Finn, Chelsea and Gupta, Abhishek and Levine, Sergey},
  booktitle = {2024 IEEE International Conference on Robotics and Automation (ICRA)},
  pages     = {16961--16969},
  year      = {2024},
  publisher = {IEEE},
  doi       = {10.1109/ICRA57147.2024.10610040}
}

@article{luo2025hilserl,
  title   = {Precise and Dexterous Robotic Manipulation via Human-in-the-Loop Reinforcement Learning},
  author  = {Luo, Jianlan and Xu, Charles and Wu, Jeffrey and Levine, Sergey},
  journal = {Science Robotics},
  volume  = {10},
  number  = {105},
  pages   = {eads5033},
  year    = {2025},
  doi     = {10.1126/scirobotics.ads5033}
}

@inproceedings{ball2023rlpd,
  title     = {Efficient Online Reinforcement Learning with Offline Data},
  author    = {Ball, Philip J. and Smith, Laura and Kostrikov, Ilya and Levine, Sergey},
  booktitle = {Proceedings of the 40th International Conference on Machine Learning},
  series    = {Proceedings of Machine Learning Research},
  volume    = {202},
  pages     = {1577--1594},
  year      = {2023},
  publisher = {PMLR}
}

@inproceedings{nakamoto2023calql,
  title     = {{Cal-QL}: Calibrated Offline {RL} Pre-Training for Efficient Online Fine-Tuning},
  author    = {Nakamoto, Mitsuhiko and Zhai, Yuexiang and Singh, Anikait and Mark, Max Sobol and Ma, Yi and Finn, Chelsea and Kumar, Aviral and Levine, Sergey},
  booktitle = {Advances in Neural Information Processing Systems},
  volume    = {36},
  pages     = {62244--62269},
  year      = {2023}
}

@inproceedings{wagenmaker2025dsrl,
  title     = {Steering Your Diffusion Policy with Latent Space Reinforcement Learning},
  author    = {Wagenmaker, Andrew and Zhang, Yunchu and Nakamoto, Mitsuhiko and Park, Seohong and Yagoub, Waleed and Nagabandi, Anusha and Gupta, Abhishek and Levine, Sergey},
  booktitle = {Proceedings of The 9th Conference on Robot Learning},
  series    = {Proceedings of Machine Learning Research},
  volume    = {305},
  pages     = {258--282},
  year      = {2025},
  publisher = {PMLR}
}

@article{li2025grrl,
  title   = {{GR-RL}: Going Dexterous and Precise for Long-Horizon Robotic Manipulation},
  author  = {Li, Yunfei and Ma, Xiao and Xu, Jiafeng and Cui, Yu and Cui, Zhongren and Han, Zhigang and Huang, Liqun and Kong, Tao and Liu, Yuxiao and Niu, Hao and Peng, Wanli and Qiao, Jingchao and Ren, Zeyu and Shi, Haixin and Su, Zhi and Tian, Jiawen and Xiao, Yuyang and Zhang, Shenyu and Zheng, Liwei and Li, Hang and Wu, Yonghui},
  journal = {arXiv preprint arXiv:2512.01801},
  year    = {2025},
  eprint  = {2512.01801},
  archivePrefix = {arXiv},
  primaryClass  = {cs.RO}
}

@inproceedings{zhou2025efficient,
  title     = {Efficient Online Reinforcement Learning Fine-Tuning Need Not Retain Offline Data},
  author    = {Zhou, Zhiyuan and Peng, Andy and Li, Qiyang and Levine, Sergey and Kumar, Aviral},
  booktitle = {International Conference on Learning Representations},
  volume    = {2025},
  pages     = {32343--32368},
  year      = {2025}
}

@inproceedings{tan2026robodopamine,
  title={Robo-Dopamine: General Process Reward Modeling for High-Precision Robotic Manipulation},
  author={Tan, Huajie and Chen, Sixiang and Xu, Yijie and Wang, Zixiao and Ji, Yuheng and Chi, Cheng and Lyu, Yaoxu and Zhao, Zhongxia and Chen, Xiansheng and Co, Peterson and Xie, Shaoxuan and Yao, Guocai and Wang, Pengwei and Wang, Zhongyuan and Zhang, Shanghang},
  booktitle={Proceedings of the IEEE/CVF Conference on Computer Vision and Pattern Recognition},
  year={2026}
}

@inproceedings{liang2026robometer,
  title={Robometer: Scaling General-Purpose Robotic Reward Models via Trajectory Comparisons},
  author={Liang, Anthony and Korkmaz, Yigit and Zhang, Jiahui and Hwang, Minyoung and Anwar, Abrar and Kaushik, Sidhant and Shah, Aditya and Huang, Alex S. and Zettlemoyer, Luke and Fox, Dieter and Xiang, Yu and Li, Anqi and Bobu, Andreea and Gupta, Abhishek and Tu, Stephen and Biyik, Erdem and Zhang, Jesse},
  booktitle={Robotics: Science and Systems},
  year={2026}
}

@inproceedings{zhang2025rewind,
  title={ReWiND: Language-Guided Rewards Teach Robot Policies without New Demonstrations},
  author={Zhang, Jiahui and Luo, Yusen and Anwar, Abrar and Sontakke, Sumedh Anand and Lim, Joseph J and Thomason, Jesse and Biyik, Erdem and Zhang, Jesse},
  booktitle={Proceedings of the 9th Conference on Robot Learning},
  volume={305},
  pages={460--488},
  year={2025}
}

@article{frans2025diffusion,
  title={Diffusion Guidance Is a Controllable Policy Improvement Operator},
  author={Frans, Kevin and Park, Seohong and Abbeel, Pieter and Levine, Sergey},
  journal={arXiv preprint arXiv:2505.23458},
  year={2025}
}

@article{deng2026e2hil,
  title={E2HiL: Entropy-Guided Sample Selection for Efficient Real-World Human-in-the-Loop Reinforcement Learning},
  author={Deng, Haoyuan and Lin, Yudong and Xue, Yuanjiang and Du, Haoyang and Wang, Qianzhun and Zhou, Boyang and Wu, Zhenyu and Wang, Ziwei},
  journal={IEEE Robotics and Automation Letters},
  volume={11},
  number={7},
  pages={8084--8091},
  year={2026},
  doi={10.1109/LRA.2026.3693577}
}

@inproceedings{luo2025humanagent,
  title={Human-Agent Joint Learning for Efficient Robot Manipulation Skill Acquisition},
  author={Luo, Shengcheng and Peng, Quanquan and Lv, Jun and Hong, Kaiwen and Driggs-Campbell, Katherine Rose and Lu, Cewu and Li, Yong-Lu},
  booktitle={2025 IEEE International Conference on Robotics and Automation (ICRA)},
  pages={1370--1377},
  year={2025},
  publisher={IEEE},
  doi={10.1109/ICRA55743.2025.11127637}
}

@article{peng2019advantage,
  title={Advantage-Weighted Regression: Simple and Scalable Off-Policy Reinforcement Learning},
  author={Peng, Xue Bin and Kumar, Aviral and Zhang, Grace and Levine, Sergey},
  journal={arXiv preprint arXiv:1910.00177},
  year={2019}
}

@article{nair2020awac,
  title={{AWAC}: Accelerating Online Reinforcement Learning with Offline Datasets},
  author={Nair, Ashvin and Gupta, Abhishek and Dalal, Murtaza and Levine, Sergey},
  journal={arXiv preprint arXiv:2006.09359},
  year={2020}
}

@inproceedings{ren2025diffusion,
  title={Diffusion Policy Policy Optimization},
  author={Ren, Allen Z. and Lidard, Justin and Ankile, Lars Lien and Simeonov, Anthony and Agrawal, Pulkit and Majumdar, Anirudha and Burchfiel, Benjamin and Dai, Hongkai and Simchowitz, Max},
  booktitle={International Conference on Learning Representations},
  year={2025}
}

@article{ho2022classifierfree,
  title={Classifier-Free Diffusion Guidance},
  author={Ho, Jonathan and Salimans, Tim},
  journal={arXiv preprint arXiv:2207.12598},
  year={2022}
}

@inproceedings{bellemare2017distributional,
  title={A Distributional Perspective on Reinforcement Learning},
  author={Bellemare, Marc G. and Dabney, Will and Munos, R{\'e}mi},
  booktitle={Proceedings of the 34th International Conference on Machine Learning},
  volume={70},
  pages={449--458},
  year={2017},
  publisher={PMLR}
}

@article{black2025pi05,
  title={${\pi}_{0.5}$: A Vision-Language-Action Model with Open-World Generalization},
  author={Black, Kevin and Brown, Noah and Darpinian, James and Dhabalia, Karan and Driess, Danny and Esmail, Adnan and Equi, Michael and Finn, Chelsea and Fusai, Niccolo and Galliker, Manuel Y. and Ghosh, Dibya and Groom, Lachy and Hausman, Karol and Ichter, Brian and Jakubczak, Szymon and Jones, Tim and Ke, Liyiming and LeBlanc, Devin and Levine, Sergey and Li-Bell, Adrian and Mothukuri, Mohith and Nair, Suraj and Pertsch, Karl and Ren, Allen Z. and Shi, Lucy Xiaoyang and Smith, Laura and Springenberg, Jost Tobias and Stachowicz, Kyle and Tanner, James and Vuong, Quan and Walke, Homer and Walling, Anna and Wang, Haohuan and Yu, Lili and Zhilinsky, Ury},
  journal={arXiv preprint arXiv:2504.16054},
  year={2025}
}

@article{torne2026mem,
  title={{MEM}: Multi-Scale Embodied Memory for Vision Language Action Models},
  author={Torne, Marcel and Pertsch, Karl and Walke, Homer and Vedder, Kyle and Nair, Suraj and Ichter, Brian and Ren, Allen Z. and Wang, Haohuan and Tang, Jiaming and Stachowicz, Kyle and Dhabalia, Karan and Equi, Michael and Vuong, Quan and Springenberg, Jost Tobias and Levine, Sergey and Finn, Chelsea and Driess, Danny},
  journal={arXiv preprint arXiv:2603.03596},
  year={2026}
}
\bibliographystyle{iclr2027_conference}

\clearpage
\ifdefined\maintextonly
  \hypersetup{bookmarksdepth=-2}
  \makeatletter
  \let\mainonlyoriginaldestination\new@pdflink
  \def\new@pdflink#1{%
    \immediate\write\@auxout{\string\mainonlyappendixdestination{#1}}%
    \mainonlyoriginaldestination{#1}}
  \long\def\protected@write#1#2#3{%
    \begingroup
      #2\let\protect\@unexpandable@protect
      \edef\reserved@a{\immediate\write#1{#3}}\reserved@a
    \endgroup
    \if@nobreak\ifvmode\nobreak\fi\fi}
  \makeatother
  \AddToHook{shipout/before}{\DiscardShipoutBox}
\fi
\appendix
\raggedbottom
\section{Experimental Setup}
\label{app:experiment_setup}

Figure~\ref{fig:real_robot_setup} summarizes the real-robot platforms and task workspaces, while Table~\ref{tab:task_data_budgets} reports the episode counts used for offline demonstrations, offline RL warmup, and each Online RL collection round.

\subsection{Task Procedures}
\label{app:task_procedures}

\paragraph{\textit{Box Assembly}.}
A cardboard box is placed in a random pose on the table.
The robot grasps one of its front flaps to bring the box upright.
The other arm then slowly folds a side flap inward and clamps it in place.
Finally, the robot lays the box flat and repeats the inward fold-and-clamp operation on the opposite side.
We evaluate the task on four tablecloths that differ in both texture and appearance.
The task is successful if, within the 120~s time limit, both side flaps are fully locked in place and the box lies horizontally on the table.

\paragraph{\textit{Laundry Folding}.}
A crumpled garment is randomly arranged at a corner of the table.
The robot first moves it into its workspace.
It repeatedly locates the garment seams and progressively spreads the garment flat.
If necessary, the robot rotates the garment to a horizontal orientation.
It folds both sleeves inward, rotates the garment back to a vertical orientation, and folds it in half.
It then pushes the folded garment to the table corner.
\review{The task is successful if the robot completes the procedure within 240~s and pushes the garment to the table corner while keeping it folded.}

\paragraph{\textit{Magnetic Bus Assembly}.}
The robot starts with six magnetic blocks scattered across the workspace and first identifies the three base blocks of the bus.
It sequentially aligns and magnetically joins their narrow side edges to form a horizontal base.
The robot then identifies the three body blocks labeled ``V'', ``A'', and ``M'' and attaches them magnetically to the base in that order, placing each block flush against its neighboring body block.
The task is successful if, within 300~s, all six blocks form the bus, with a connected horizontal base and tightly adjoining body blocks spelling ``VAM''.

\begin{figure}[H]
    \centering
    \includegraphics[width=0.60\textwidth]{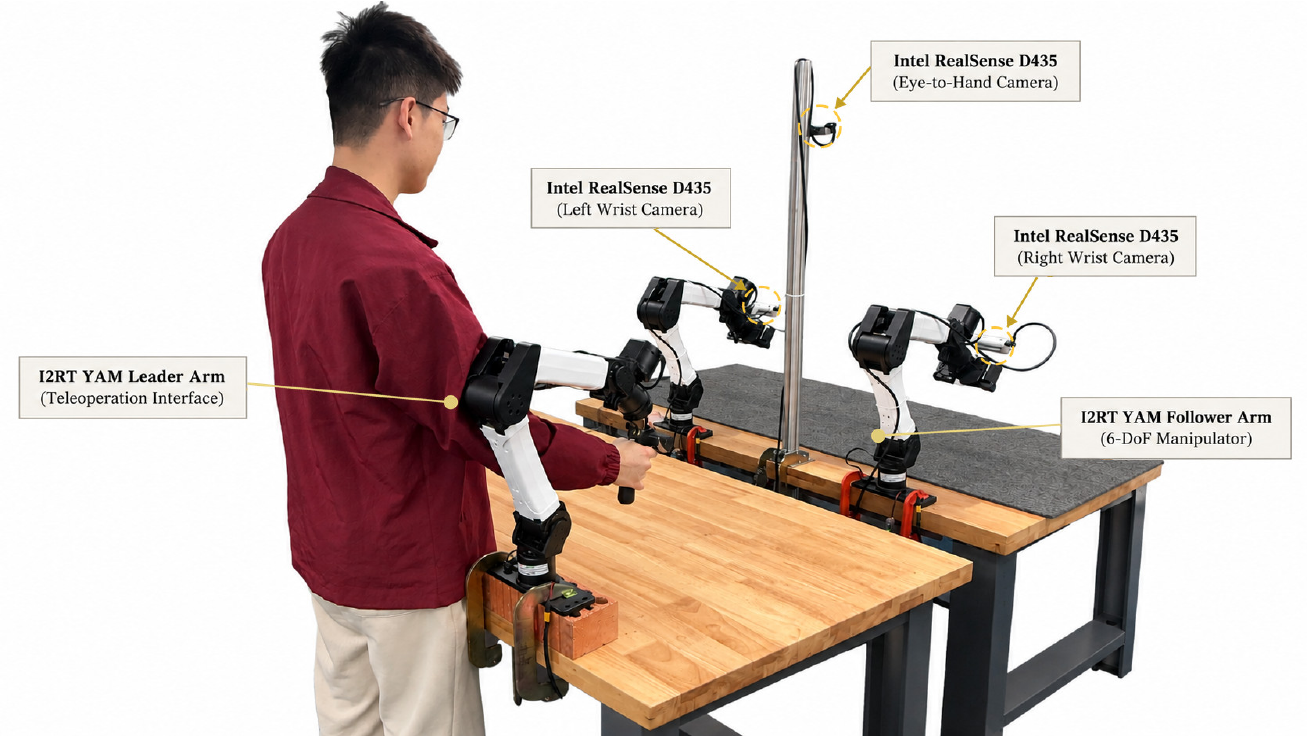}
    \caption{\textbf{Real-robot experimental setup.}
    The illustration summarizes the robot hardware, sensing configuration, human-intervention interface, and task workspaces used in our experiments.}
    \label{fig:real_robot_setup}
\end{figure}

\begin{table}[H]
    \centering
    \caption{\review{\textbf{Task-specific data-collection budgets.}}}
    \label{tab:task_data_budgets}
    \small
    \setlength{\tabcolsep}{5.5pt}
    \begin{tabular}{@{}lccc@{}}
        \toprule
        \textbf{Task} & \review{\shortstack{\textbf{Offline demonstrations}\\\textbf{(episodes)}}} & \review{\shortstack{\textbf{Offline RL warmup}\\\textbf{(episodes)}}} & \review{\shortstack{\textbf{Online RL rollouts per round}\\\textbf{(episodes)}}} \\
        \midrule
        \textit{Box Assembly} & 200 & 100 & 20 \\
        \textit{Laundry Folding} & 90 & 50 & 15 \\
        \textit{Magnetic Bus Assembly} & 30 & 40 & 10 \\
        \bottomrule
    \end{tabular}
\end{table}

\paragraph{Matched training budgets.}
\review{For each task, the \textit{Offline RL} and \textit{Online RL} comparisons use the same total gradient-step budget and the same data-collection budget within each stage.}
\review{\textit{Offline RL} uses batched updates, whereas \textit{Online RL} distributes its updates across repeated collection-and-training iterations.}

\begingroup
\reviewcolor
\Needspace{9\baselineskip}
\section{Additional Quantitative Evidence}

\subsection{Value-Expert Replacement at Inference}
\label{app:value_replacement}

Table~\ref{tab:online_rl_final_three} reports value-expert replacement across all three tasks with the action experts fixed at the final three LfE checkpoints.
The demonstration-initialized value expert is trained on $\mathcal{D}_{\mathrm{demo}}$ using the progress objective, and the post-warmup value expert is the snapshot after offline RL warmup.
Both conditions preserve the value caching mechanism and use the same evaluation episode counts as the main Online DAgger and LfE experiments on \textit{Box Assembly} and \textit{Laundry Folding}.
The original LfE value reference reuses the main-experiment results.

\subsection{Sensitivity to Checkpoint Aggregation}
\label{app:checkpoint_sensitivity}

We additionally average success rates across every available LfE checkpoint in Fig.~\ref{fig:online_rl_performance} to assess dependence on the final-three-checkpoint summary.
Table~\ref{tab:lfe_checkpoint_sensitivity} reports an arithmetic mean over checkpoints within each task and method.
VaM exceeds the caching ablation on each task, with a mean gap of $31.3\%$, compared with $29.4\%$ over the final three checkpoints.
The consistent gains across aggregation windows support the robustness of the conclusions drawn from the final three checkpoints.

\begin{table}[H]
    \reviewcolor
    \centering
    \caption{\textbf{Success rate over all available LfE checkpoints.}
    $N$ is the number of checkpoints per method.}
    \label{tab:lfe_checkpoint_sensitivity}
    \small
    \begin{tabular}{@{}lcccc@{}}
        \toprule
        \textbf{Task} & $N$ & \textbf{VaM} $(\%)$ & \textbf{VaM (w/o value caching)} $(\%)$ & \textbf{Gap} $(\%)$ \\
        \midrule
        \reviewtablerows{exp_fig/analysis/review-evidence/outputs/checkpoint-rows.tex}
        \bottomrule
    \end{tabular}
\end{table}

\subsection{Recovery-Event Counts}
\label{app:recovery_counts}

Human annotators review all LfE rollouts from the four methods across all three tasks without knowing which method produced each rollout.
Table~\ref{tab:recovery_event_counts} reports the number of successful recoveries, the number of events requiring recovery, and the number of episodes reviewed for each method.
A successful recovery requires the policy to execute corrective actions and complete the affected task stage.
Multiple events may occur within one episode, so event counts and episode counts represent different units.
The reported rates are event-level fractions.

\begin{table}[H]
    \reviewcolor
    \centering
    \caption{\textbf{Raw counts for the blind recovery evaluation.}}
    \label{tab:recovery_event_counts}
    \small
    \setlength{\tabcolsep}{3pt}
    \begin{tabular}{@{}ccccc@{}}
        \toprule
        \textbf{Task} & \textbf{Method} & \shortstack{\textbf{\# recovered /}\\\textbf{\# required}} & \textbf{\# episodes} & \shortstack{\textbf{Recovery rate}\\$(\%)$} \\
        \midrule
        \reviewtablerows{exp_fig/analysis/review-evidence/outputs/recovery-rows.tex}
        \bottomrule
    \end{tabular}
\end{table}

\subsection{Total Experimental Time}
\label{app:time_accounting}

Table~\ref{tab:recorded_time_components} reports experimental time per gradient step and total experimental time.
\textit{Offline RL} first collects data, then trains the value expert and subsequently the action expert.
\textit{Online RL} runs data collection, value training, and action training concurrently.

\begin{table}[H]
    \reviewcolor
    \centering
    \caption{\textbf{Experimental time per gradient step and total experimental time.}
    Time per gradient step is reported as mean $\pm$ sample standard deviation.}
    \label{tab:recorded_time_components}
    \small
    \begin{tabular}{@{}llcc@{}}
        \toprule
        \textbf{Task} & \textbf{Setting} & \shortstack{\textbf{Time per gradient step}\\\textbf{(s)}} & \shortstack{\textbf{Total experimental time}\\\textbf{(hours)}} \\
        \midrule
        \reviewtablerows{exp_fig/analysis/review-evidence/outputs/timing-rows.tex}
        \bottomrule
    \end{tabular}
\end{table}

\Needspace{8\baselineskip}
\subsection{Hyperparameters}
\label{app:hyperparameters}

Table~\ref{tab:hyperparameters} summarizes the hyperparameters for value and action learning, including task-specific settings.
The observation-history stride specifies the interval in environment steps between sampled observations.
Update counts are measured in gradient steps.

\begin{table}[H]
    \reviewcolor
    \centering
    \caption{\textbf{Hyperparameters for value and action learning.}}
    \label{tab:hyperparameters}
    \small
    \setlength{\tabcolsep}{3pt}
    \renewcommand{\arraystretch}{1.08}
    \begin{tabular*}{\linewidth}{@{\extracolsep{\fill}}llll@{}}
        \toprule
        \textbf{Hyperparameter} & \textbf{Scope} & \textbf{Task} & \textbf{Value} \\
        \midrule
        Value-support atoms $K$ & Value expert & All tasks & 256 \\
        Action horizon $H$ (steps) & Action expert & All tasks & 100 \\
        Control frequency (Hz) & Deployment & All tasks & 100 \\
        Progress-label std. dev. $\sigma_{\mathrm{init}}$ & Value expert & All tasks & $3/512$ \\
        Target-network EMA coefficient $\beta$ & Value expert & All tasks & 1.0 \\
        Value-training stride $d$ (steps) & Value training & All tasks & 30 \\
        Advantage quantile $\rho$ & Action training & All tasks & 0.70 \\
        Observation-history horizon $H_\pi$ & Action expert & All tasks & 4 \\
        Observation-history horizon & Value expert & All tasks & 2 \\
        Observation-history stride (steps) & Value / action & All tasks & 30 \\
        Optimizer & Value / action & All tasks & AdamW \\
        Learning rate & Value / action & All tasks & $10^{-5}$ \\
        Learning-rate schedule & Value / action & All tasks & Constant \\
        Batch size (samples) & Value / action & All tasks & 64 \\
        Initialization steps & Value / action & All tasks & $40{,}000$ \\
        Offline RL warmup steps & Value / action & All tasks & $10{,}000$ \\
        \addlinespace[4pt]
        \multirow{3}{*}{Discount factor $\gamma$} & \multirow{3}{*}{Value expert} & \textit{Box Assembly} & 0.999 \\
        & & \textit{Laundry Folding} & 0.9997 \\
        & & \textit{Magnetic Bus Assembly} & 0.9997 \\
        \addlinespace[3pt]
        \multirow{3}{*}{Guidance weight $w$} & \multirow{3}{*}{Action inference} & \textit{Box Assembly} & 1.25 \\
        & & \textit{Laundry Folding} & 1.0 \\
        & & \textit{Magnetic Bus Assembly} & 1.0 \\
        \addlinespace[3pt]
        \multirow{3}{*}{Updates per round} & \multirow{3}{*}{Value / action} & \textit{Box Assembly} & 600 \\
        & & \textit{Laundry Folding} & $1{,}000$ \\
        & & \textit{Magnetic Bus Assembly} & $1{,}000$ \\
        \bottomrule
    \end{tabular*}
\end{table}
\endgroup

\raggedbottom
\Needspace{0.95\textheight}
\section{Additional Qualitative Comparisons of Value Caching}
\label{app:value_caching_qualitative}

Figure~\ref{fig:additional_value_caching} compares autonomous LfE rollouts on \textit{Laundry Folding} and \textit{Magnetic Bus Assembly}.
Across both tasks, VaM exhibits more progress-ordered PCA trajectories and recovery from stuck states, whereas its ablation shows failed execution with low or misleading value predictions.

\begin{minipage}{\textwidth}
    \setlength{\abovecaptionskip}{6pt}
    \setlength{\belowcaptionskip}{2pt}
    \centering
    \includegraphics[width=0.39\textwidth]{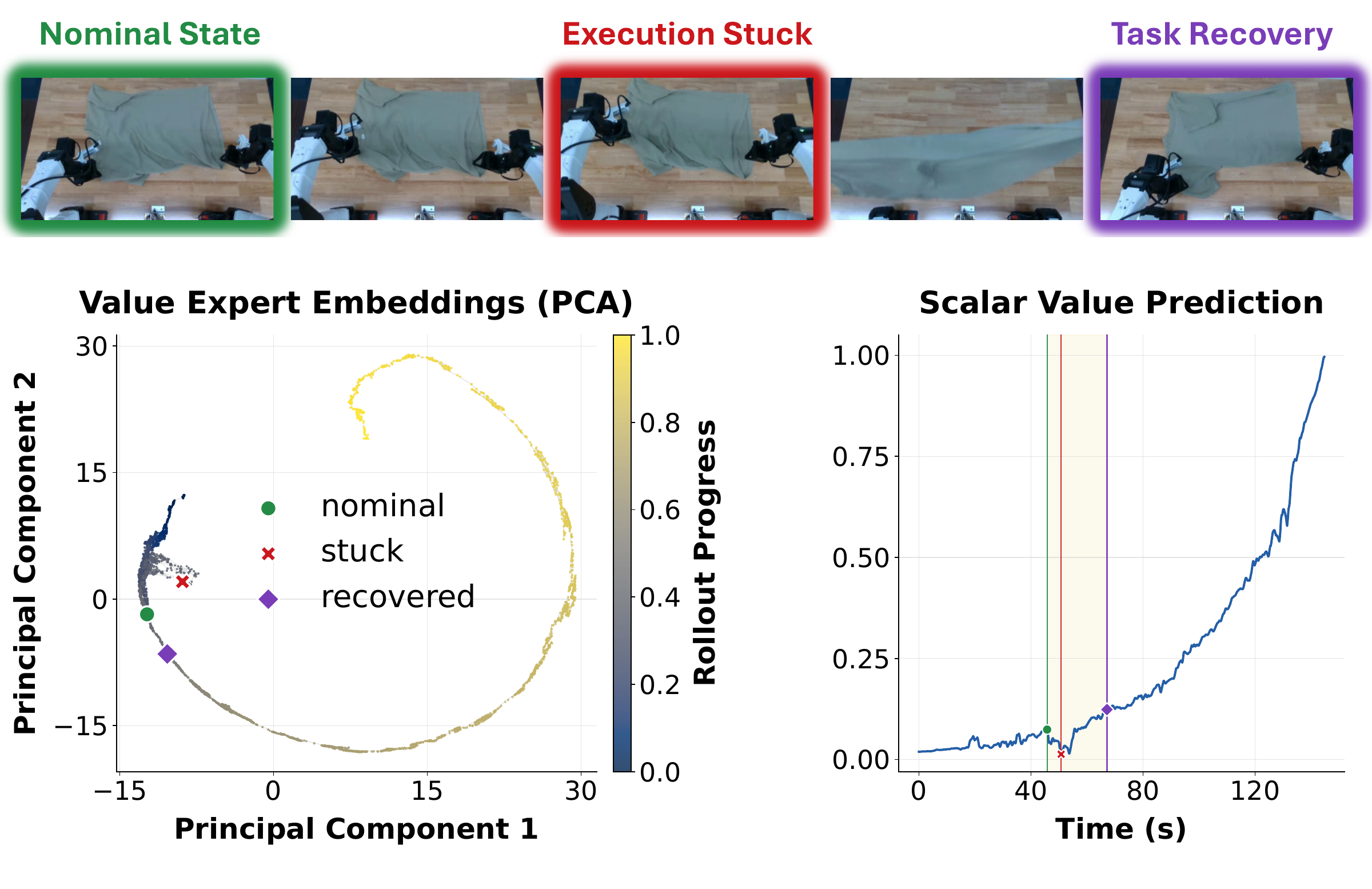}%
    \hspace{0.01\textwidth}
    \includegraphics[width=0.39\textwidth]{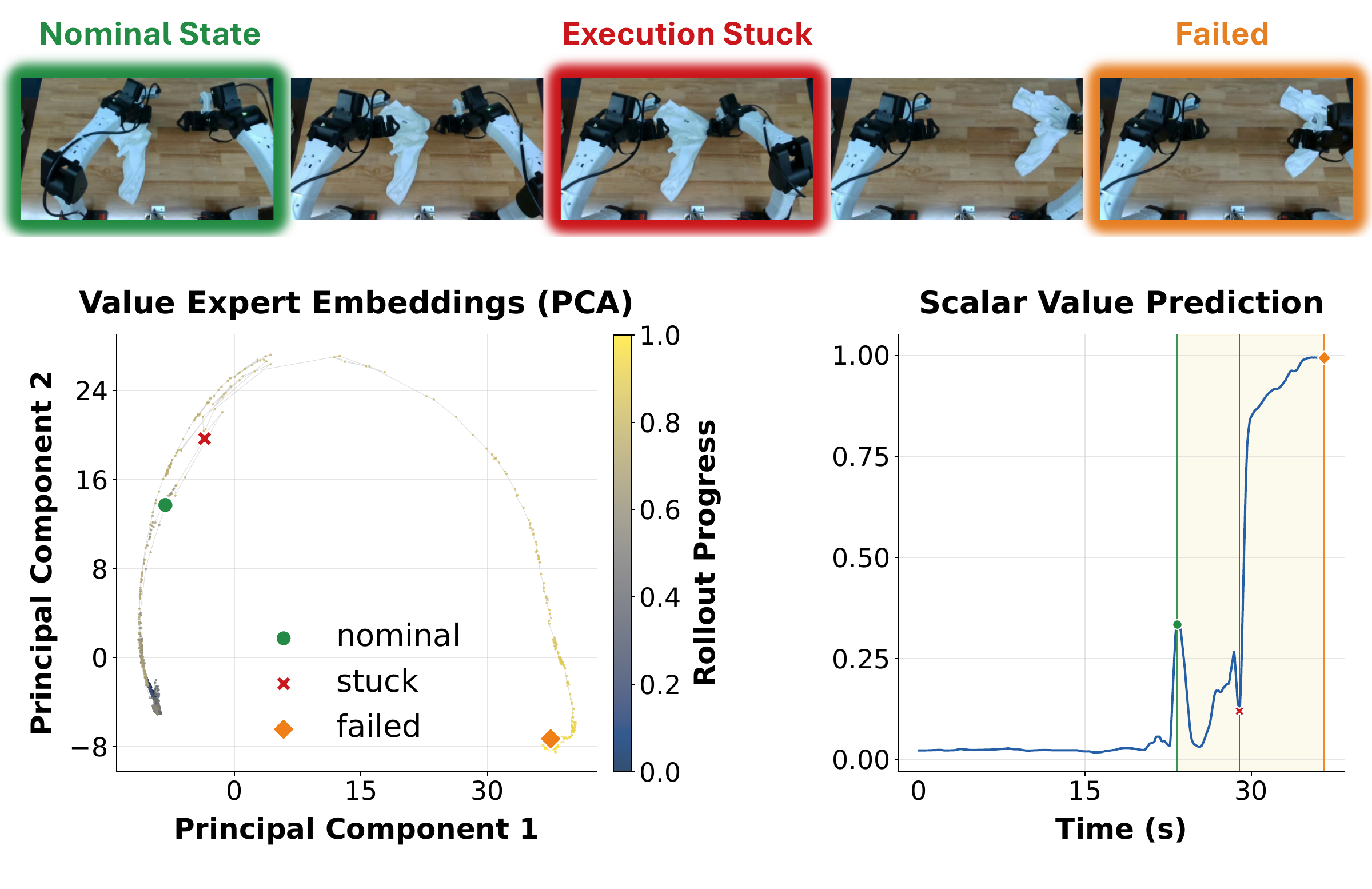}
    \par\vspace{2pt}
    \includegraphics[width=0.39\textwidth]{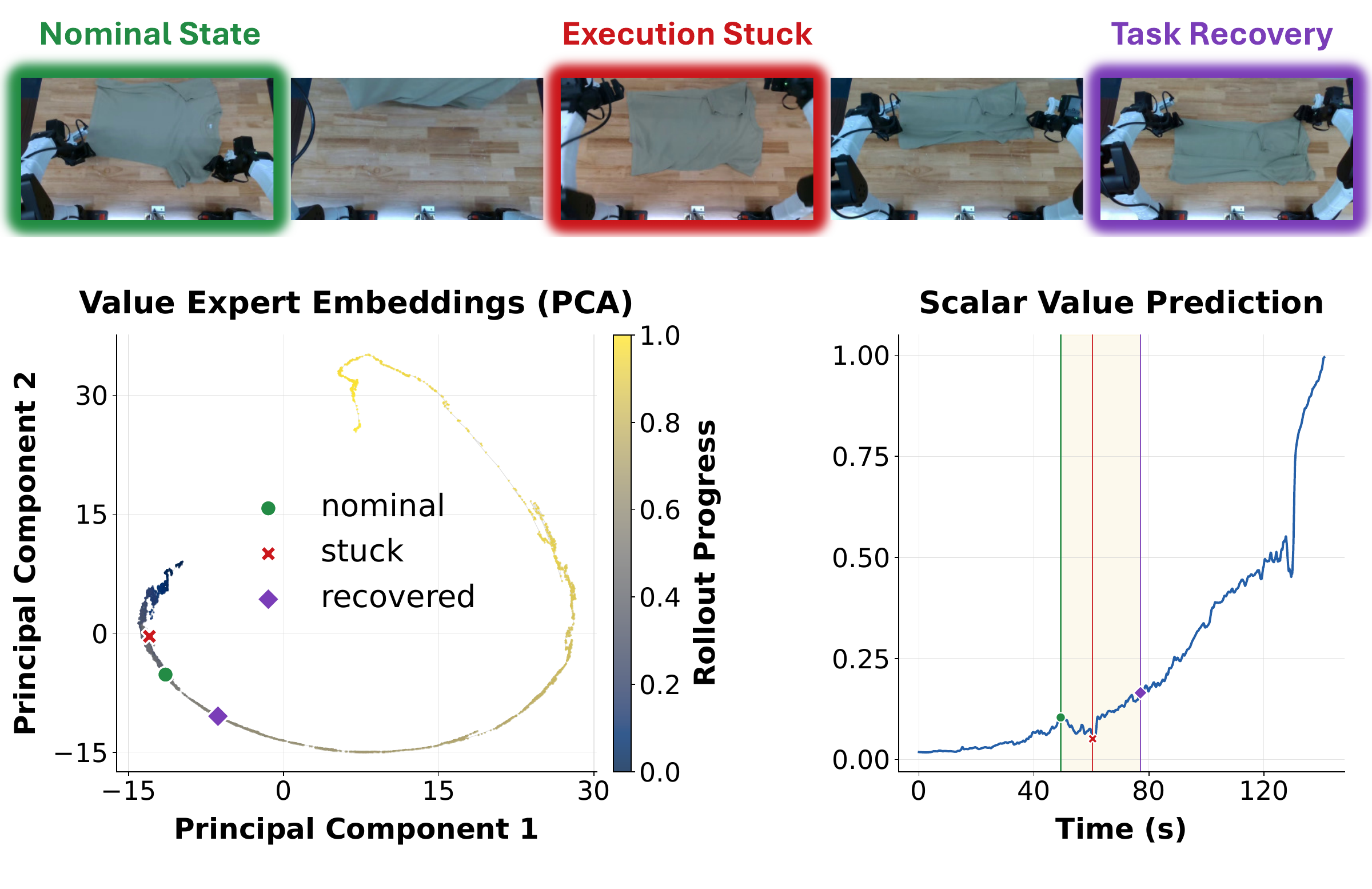}%
    \hspace{0.01\textwidth}
    \includegraphics[width=0.39\textwidth]{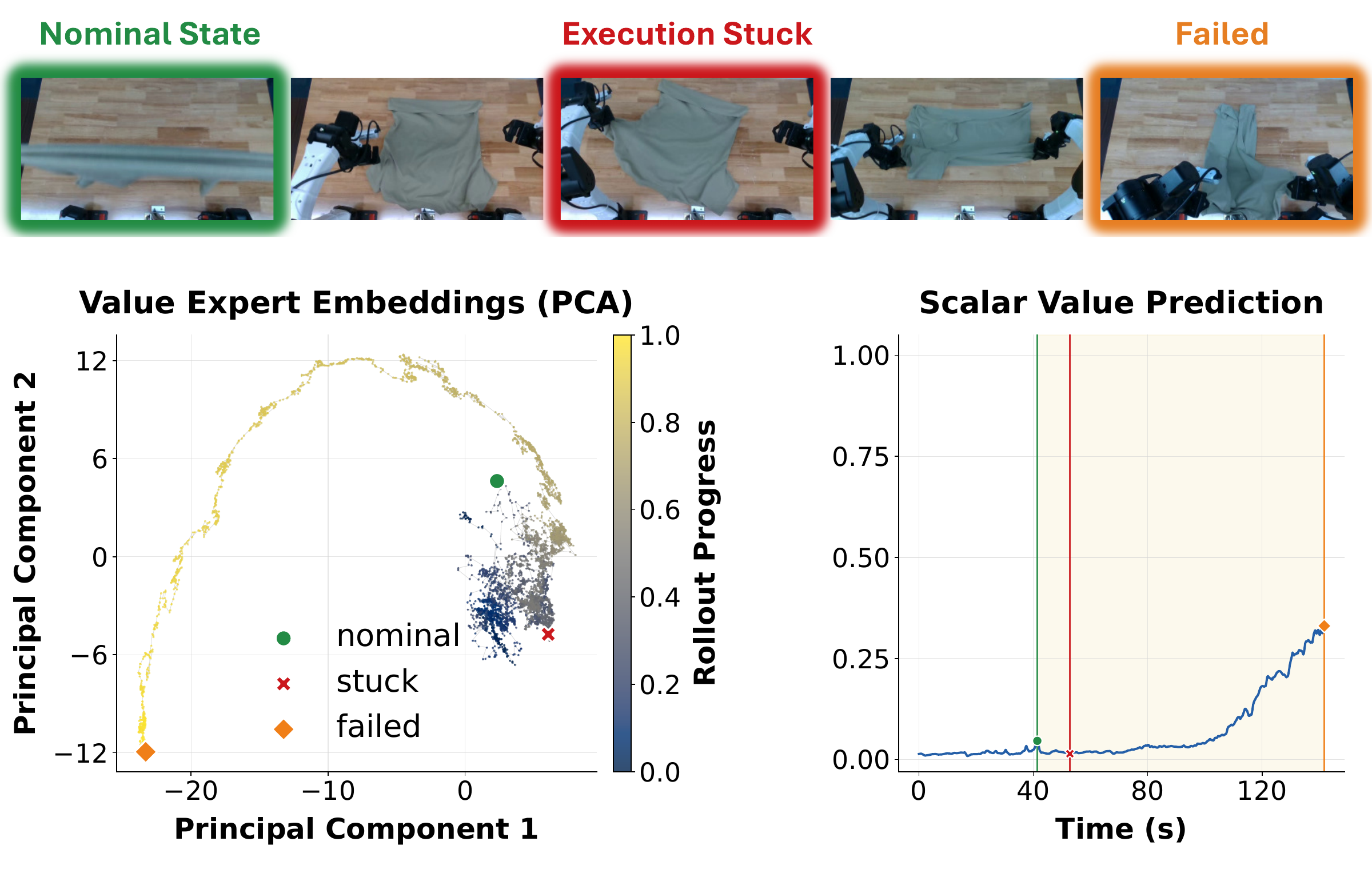}
    \par\vspace{2pt}
    \includegraphics[width=0.39\textwidth]{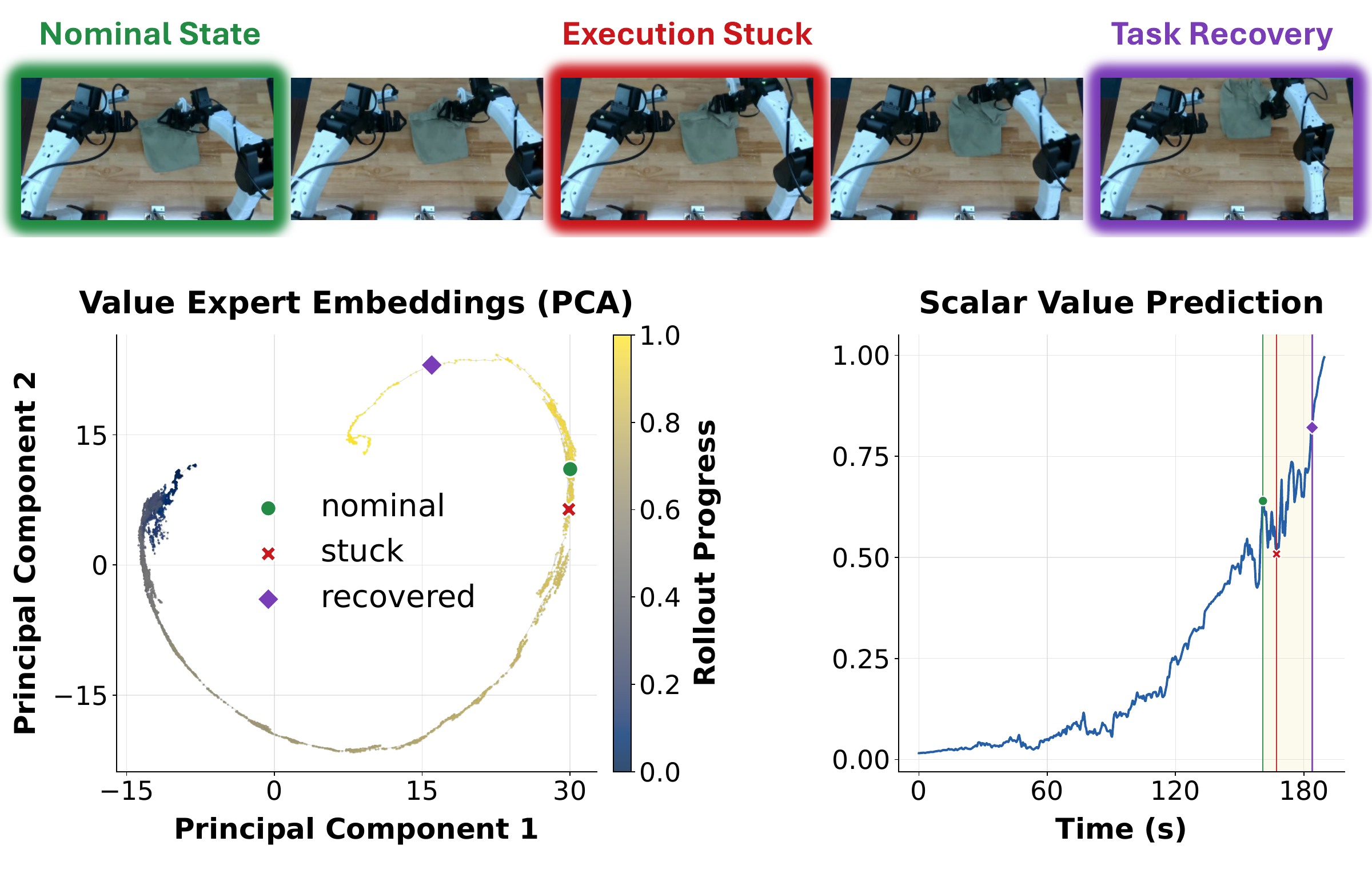}%
    \hspace{0.01\textwidth}
    \includegraphics[width=0.39\textwidth]{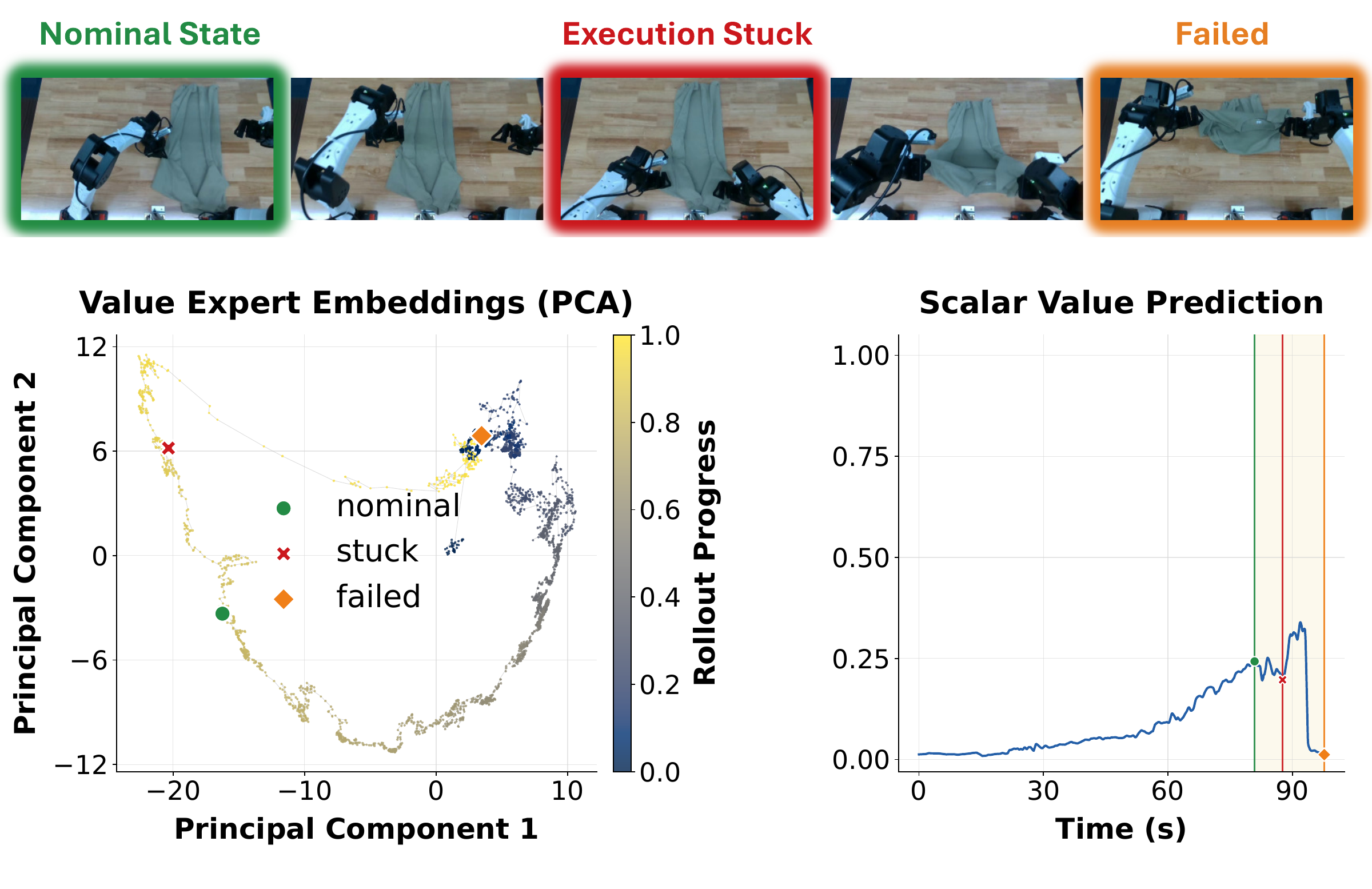}
    \par\vspace{2pt}
    \includegraphics[width=0.39\textwidth]{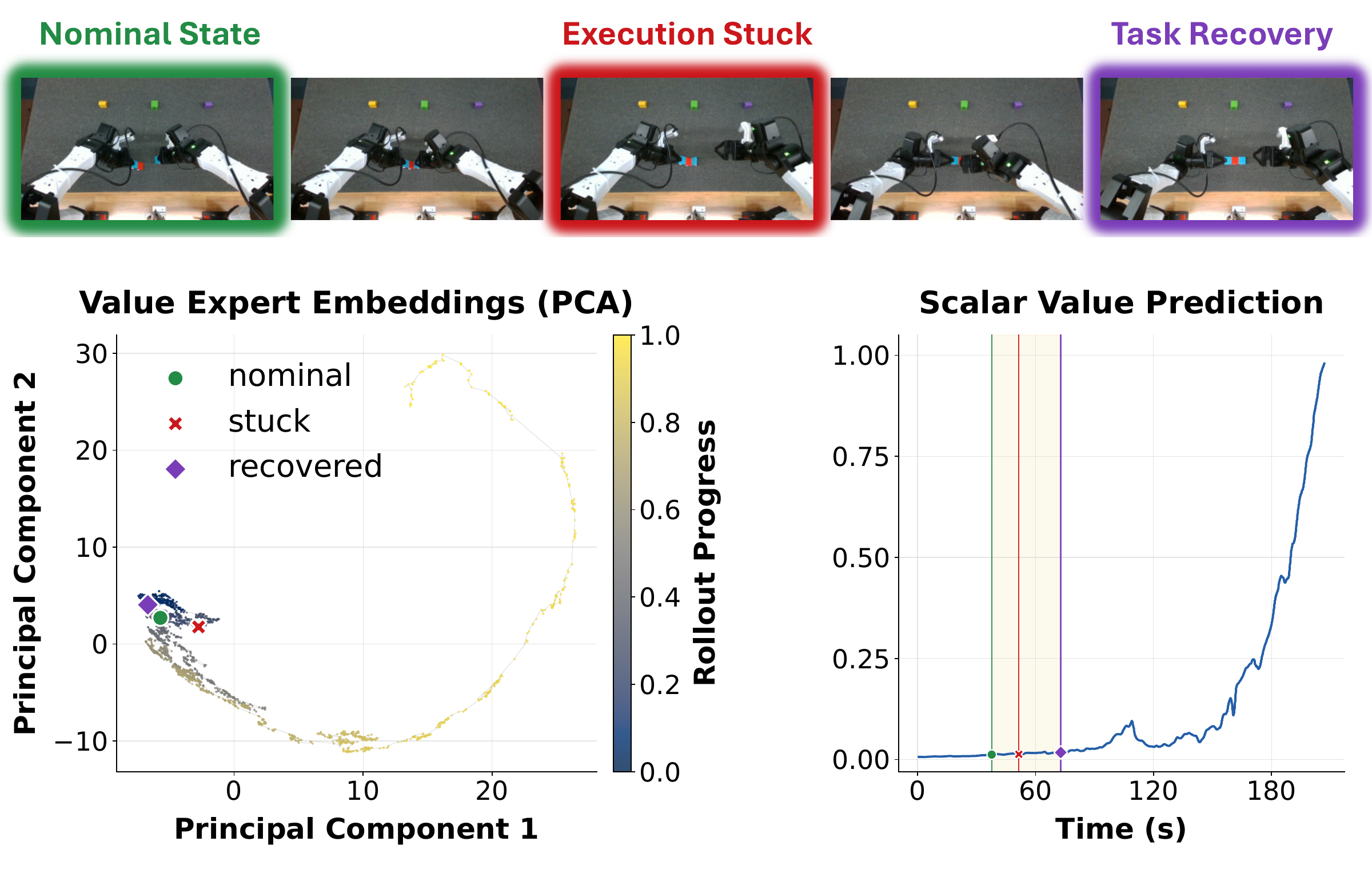}%
    \hspace{0.01\textwidth}
    \includegraphics[width=0.39\textwidth]{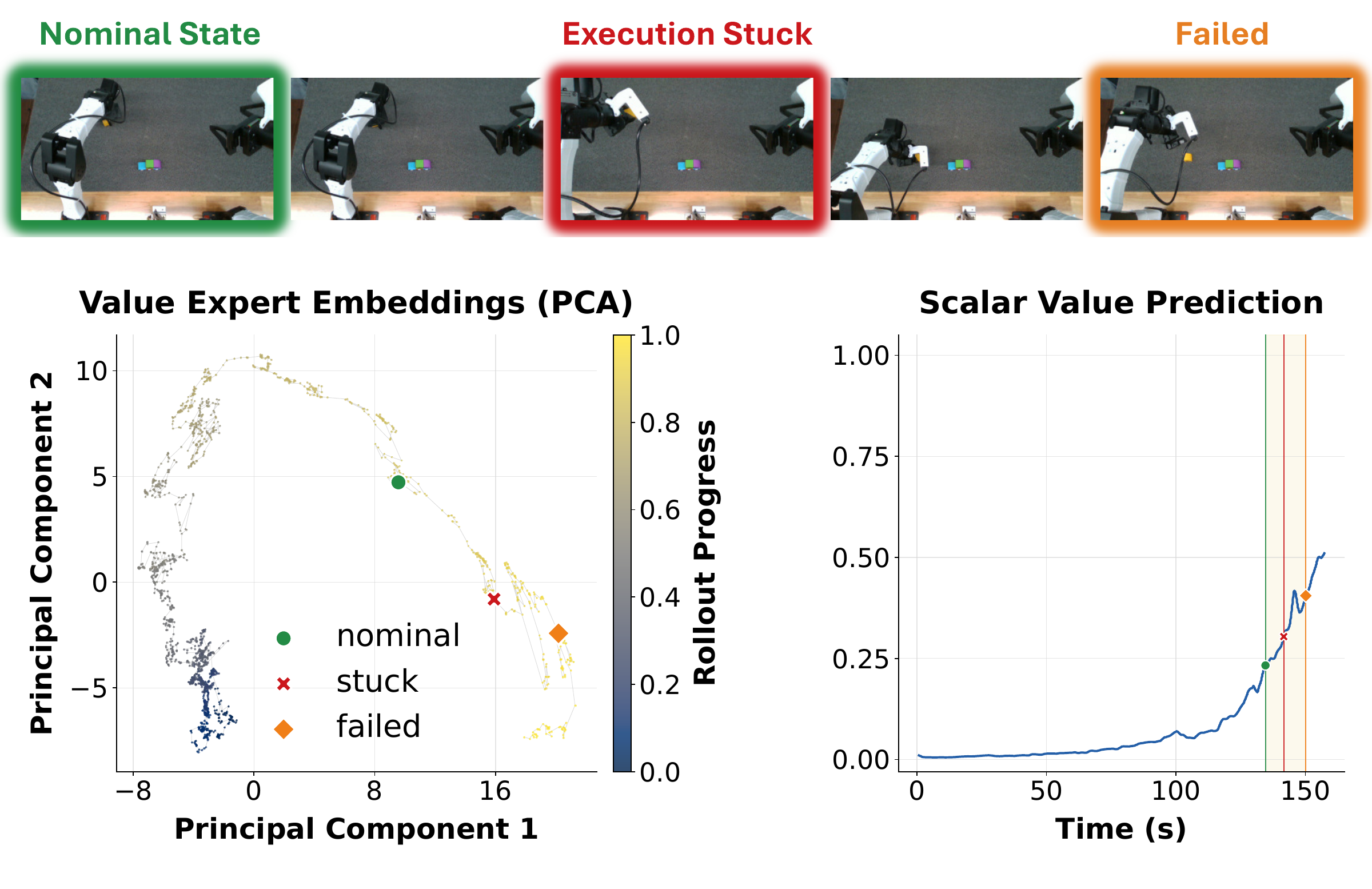}
    \par\vspace{2pt}
    \includegraphics[width=0.39\textwidth]{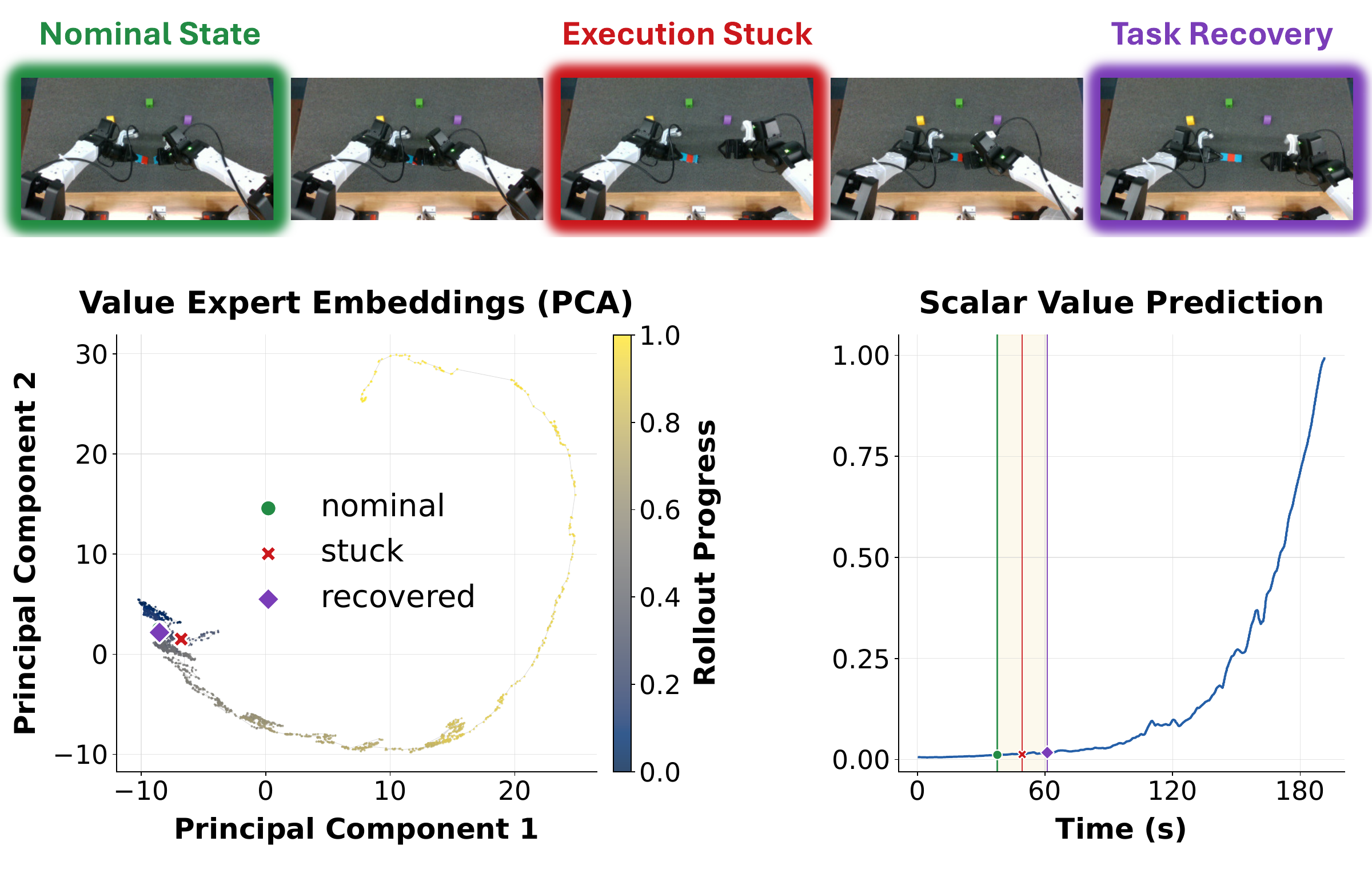}%
    \hspace{0.01\textwidth}
    \includegraphics[width=0.39\textwidth]{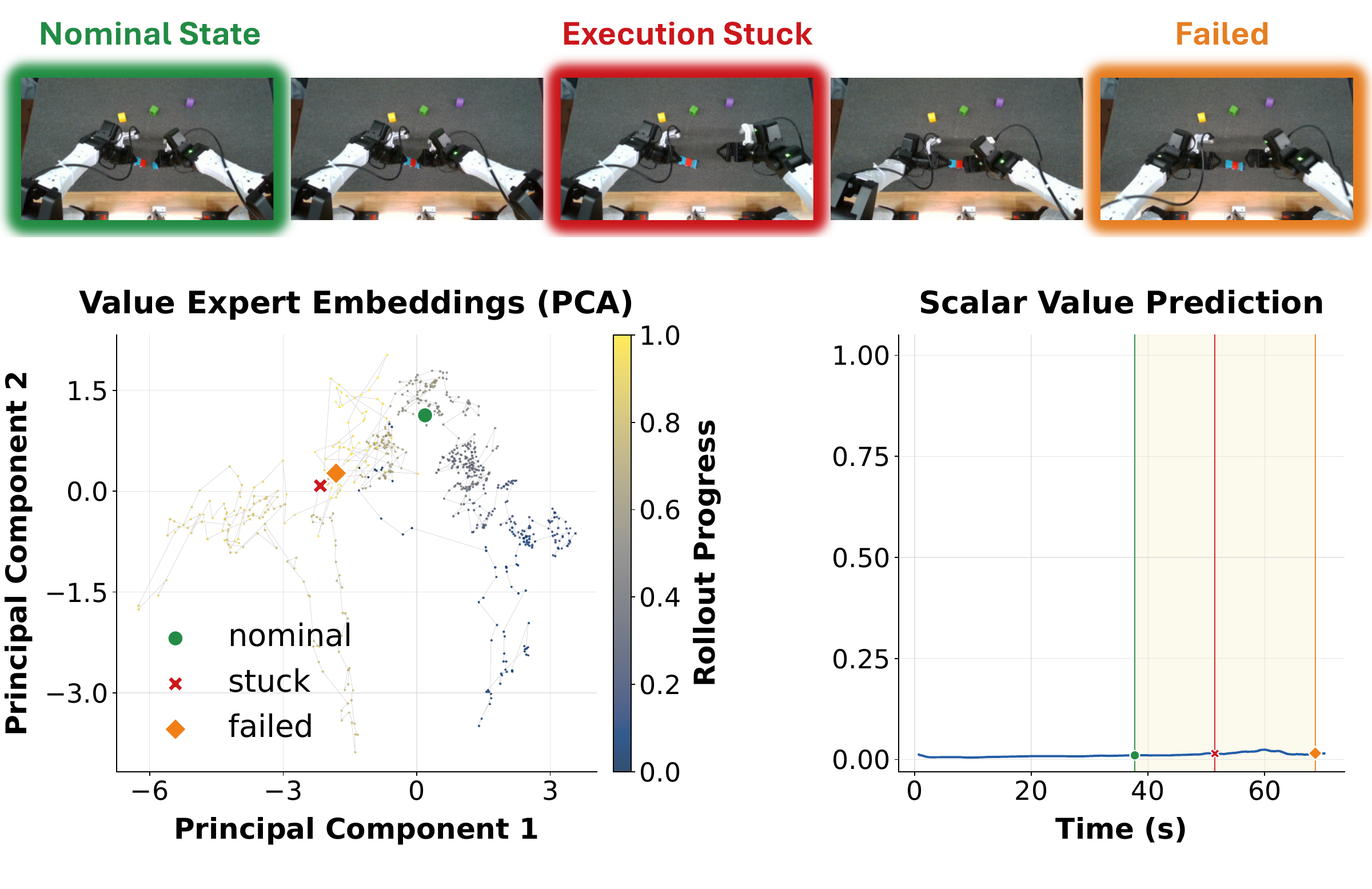}
    \nopagebreak
    \captionof{figure}{\textbf{Qualitative comparison of VaM (left) and VaM (w/o value caching) (right) on \textit{Laundry Folding} (first three rows) and \textit{Magnetic Bus Assembly} (last two rows).}
    Each example shows eye-to-hand camera images, PCA embeddings of the value expert colored by rollout progress, and scalar value predictions.
    The dark-blue trace shows the 15-frame mean of scalar value predictions.}
    \label{fig:additional_value_caching}
\par
\end{minipage}

\clearpage
\review{We further examine cached values under partial observability using two \textit{Laundry Folding} cases, following the replacement protocol in Sec.~\ref{subsec:value_caching_performance}.}
\review{Each case fixes the action expert and all observation inputs, including the four image frames at $-3$, $-2$, $-1$, and $0$ seconds, and replaces the original LfE values with demo-trained or post-warmup estimates while preserving value caching.}
\review{Figures~\ref{fig:laundry_lost_contact} and~\ref{fig:laundry_regrasped} show predicted end-effector (EEF) trajectories and gripper commands.}

\begin{minipage}{\textwidth}
    \setlength{\abovecaptionskip}{3pt}
    \setlength{\belowcaptionskip}{2pt}
    \centering
    \includegraphics[width=0.84\textwidth]{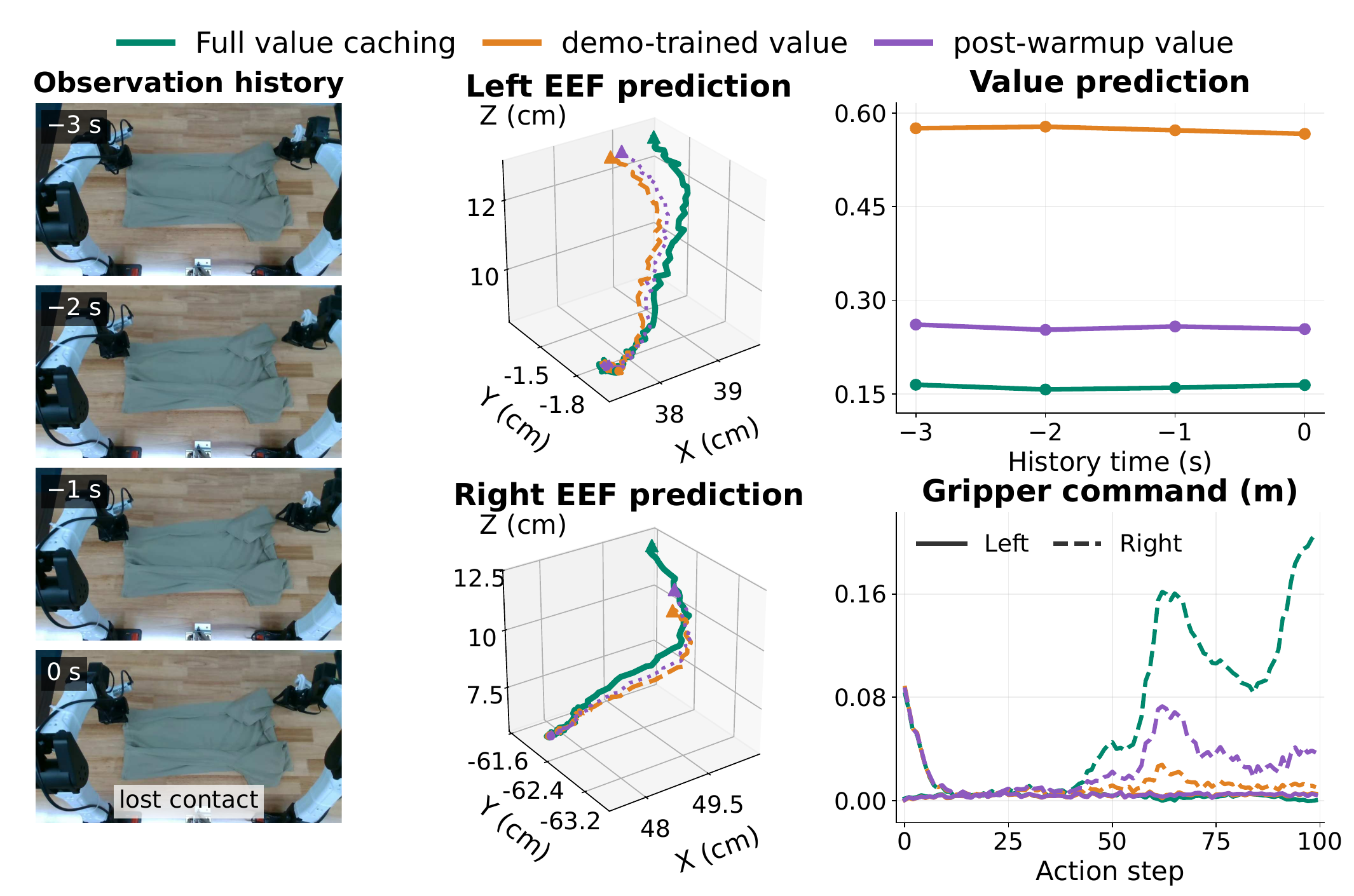}
    \captionof{figure}{\review{\textbf{Policy predictions after losing contact in \textit{Laundry Folding}.}
    Under identical visual observations, the policy using the original cached values predicts opening the right gripper wider and adjusting its pose to prepare for another attempt to grasp the garment corner, compared with using demo-trained or post-warmup values as the cache.}}
    \label{fig:laundry_lost_contact}
    \par\vspace{8pt}
    \includegraphics[width=0.84\textwidth]{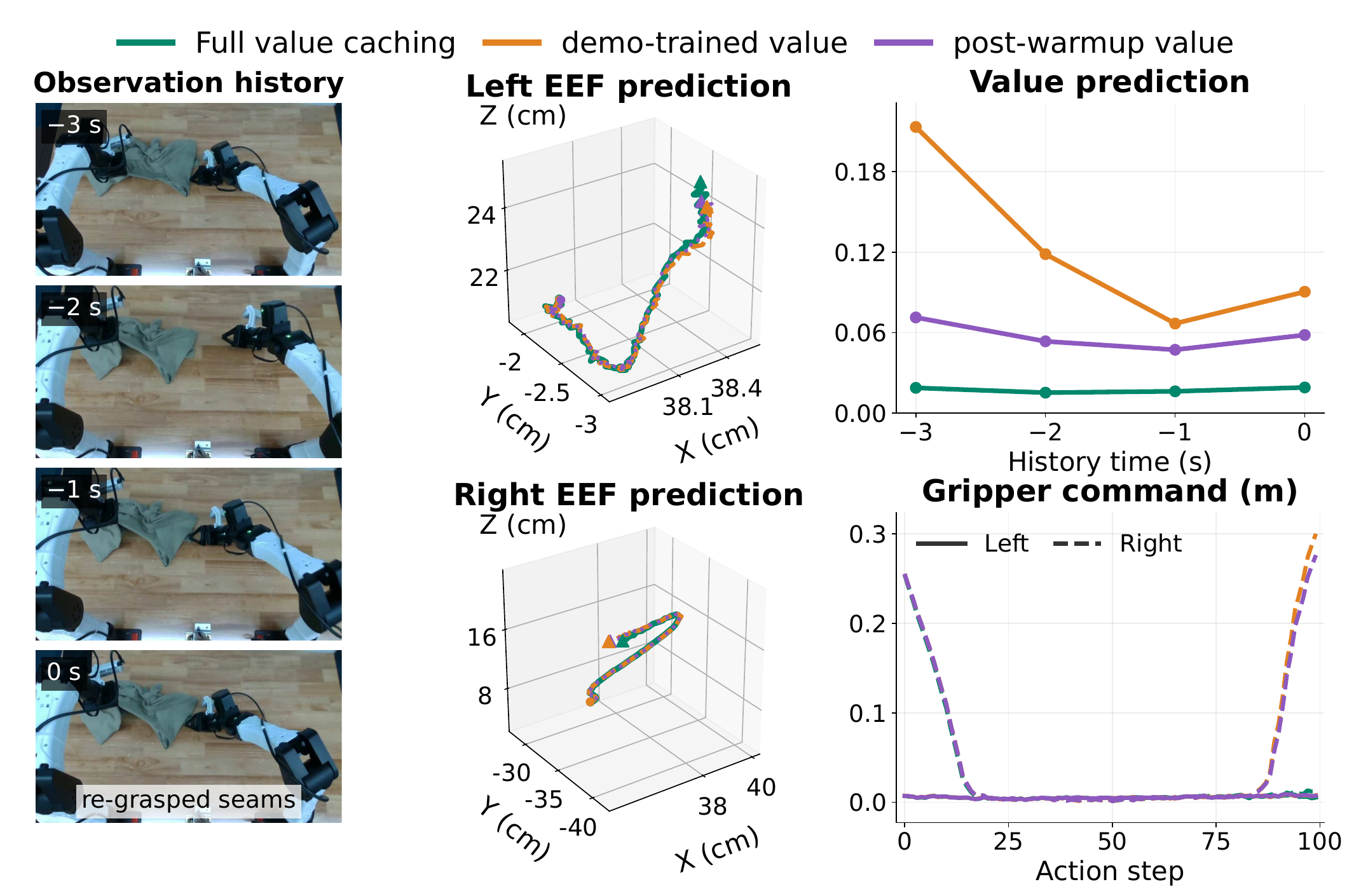}
    \captionof{figure}{\review{\textbf{Policy predictions after re-grasping the seams in \textit{Laundry Folding}.}
    The original cached values maintain the predicted right-gripper closure, while demo-trained and post-warmup values produce opening commands near the end of the action chunk, potentially causing loss of contact and leaving the garment disordered again.}}
    \label{fig:laundry_regrasped}
\par
\end{minipage}

\clearpage
\flushbottom
\section{Training Pipeline}
\label{sec:appendix}

\par\addvspace{\intextsep}
\begingroup
\hrule
\nobreak\vskip4pt
\makeatletter\fs@ruled\makeatother
\captionof{algorithm}{Continual Value-Guided Offline-to-Online RL Pipeline}
\makeatletter\unvbox\@floatcapt\makeatother
\nobreak\hrule\nobreak\vskip2pt
\phantomsection
\label{alg:training_pipeline}
\begin{algorithmic}[1]
\Require Offline expert demonstrations $\mathcal{D}_{\mathrm{demo}}$, value-data stride $d\geq1$, action-chunk/TD horizon $H$, Online DAgger rounds $N_{\mathrm{D}}$, Learn from Experience (LfE) rounds $N_{\mathrm{E}}$
\Statex \hspace{\algorithmicindent}\texttt{// Stage 1: Offline expert demonstration initialization}
\State $V_{\psi^{(0)}} \gets \Call{Init}{\mathcal{D}_{\mathrm{demo}}}$
\State $\pi_{\theta^{(0)}} \gets \Call{BC}{\mathcal{D}_{\mathrm{demo}}}$
\Statex \hspace{\algorithmicindent}\texttt{// Stage 2: Offline RL warmup}
\State $\mathcal{D}_{\mathrm{warm}}^{\mathrm{on}} \gets \Call{DeployWithIntervention}{\pi_{\theta^{(0)}},V_{\psi^{(0)}}}$
\State $\mathcal{D}_{\mathrm{warm}}^{\mathrm{pre}} \gets \Call{InterventionPrefixes}{\mathcal{D}_{\mathrm{warm}}^{\mathrm{on}}}$
\State $\widetilde{\mathcal{D}}_{\mathrm{warm}}^{V} \gets \Call{Downsample}_d(\mathcal{D}_{\mathrm{warm}}^{\mathrm{pre}})$ \Comment{retain original starts $t=0,d,2d,\ldots$ and keep each $H$-step tuple}
\State $B_{\mathrm{warm}}^{V} \gets \operatorname{Mix}_{1:1}(\mathcal{D}_{\mathrm{demo}},\widetilde{\mathcal{D}}_{\mathrm{warm}}^{V})$
\State $V_{\psi_{\mathrm{warm}}} \gets \operatorname{TD}(V_{\psi^{(0)}};B_{\mathrm{warm}}^{V})$
\State $\widehat{\mathcal{D}}_{\mathrm{warm}} \gets \Call{Annotate}{\mathcal{D}_{\mathrm{warm}}^{\mathrm{on}},V_{\psi_{\mathrm{warm}}}}$
\State $B_{\mathrm{warm}}^{\pi} \gets \operatorname{Mix}_{1:1}(\mathcal{D}_{\mathrm{demo}},\widehat{\mathcal{D}}_{\mathrm{warm}})$
\State $\pi_{\theta_{\mathrm{warm}}} \gets \operatorname{FM}(\pi_{\theta^{(0)}};B_{\mathrm{warm}}^{\pi})$
\State $(\pi_{\theta^{(0)}},V_{\psi^{(0)}}) \gets (\pi_{\theta_{\mathrm{warm}}},V_{\psi_{\mathrm{warm}}})$
\State $\mathcal{D}_{\mathrm{off}} \gets \mathcal{D}_{\mathrm{warm}}^{\mathrm{on}}$
\State $\mathcal{D}_{\mathrm{HI,off}} \gets \Call{InterventionSubset}{\mathcal{D}_{\mathrm{warm}}^{\mathrm{on}}}$
\Statex \hspace{\algorithmicindent}\texttt{// Stage 3: Online RL with Online DAgger}
\For{$k \gets 1$ to $N_{\mathrm{D}}$}
    \State $\mathcal{D}_{\mathrm{on}}^{k} \gets \Call{DeployWithIntervention}{\pi_{\theta^{(k-1)}}}$
    \State $\mathcal{D}_{\mathrm{pre}}^{k} \gets \Call{InterventionPrefixes}{\mathcal{D}_{\mathrm{on}}^{k}}$
    \State $\widetilde{\mathcal{D}}_{\mathrm{on}}^{V,k} \gets \Call{Downsample}_d(\mathcal{D}_{\mathrm{pre}}^{k})$ \Comment{same original-step sampling with TD horizon $H$}
    \State $B_{V}^{k} \gets \operatorname{Mix}_{1:1}(\mathcal{D}_{\mathrm{demo}}\cup\mathcal{D}_{\mathrm{off}},\widetilde{\mathcal{D}}_{\mathrm{on}}^{V,k})$
    \State $\widehat{\mathcal{D}}_{\mathrm{on}}^{k} \gets \Call{Annotate}{\mathcal{D}_{\mathrm{on}}^{k},V_{\psi^{(k-1)}}}$
    \State $B_{\pi}^{k} \gets \operatorname{Mix}_{1:1:1}(\mathcal{D}_{\mathrm{demo}},\mathcal{D}_{\mathrm{HI,off}},\widehat{\mathcal{D}}_{\mathrm{on}}^{k})$
    \Statex \hspace{\algorithmicindent}\textbf{Synchronously:}
    \State $V_{\psi^{(k)}} \gets \operatorname{TD}(V_{\psi^{(k-1)}};B_V^k)$
    \State $\pi_{\theta^{(k)}} \gets \operatorname{FM}(\pi_{\theta^{(k-1)}};B_{\pi}^{k})$
    \State $\mathcal{D}_{\mathrm{HI,off}} \gets \mathcal{D}_{\mathrm{HI,off}}\cup\Call{InterventionSubset}{\mathcal{D}_{\mathrm{on}}^{k}}$
    \State $\mathcal{D}_{\mathrm{off}} \gets \mathcal{D}_{\mathrm{off}}\cup\mathcal{D}_{\mathrm{on}}^{k}$
\EndFor
\Statex \hspace{\algorithmicindent}\texttt{// Stage 4: Online RL with LfE}
\For{$j \gets 1$ to $N_{\mathrm{E}}$}
    \State $k \gets N_{\mathrm{D}}+j$
    \State $\mathcal{D}_{\mathrm{on}}^{k} \gets \Call{DeployAutonomously}{\pi_{\theta^{(k-1)}}}$ with sparse terminal outcomes
    \State $\mathcal{D}_{\mathrm{on,TD}}^{k} \gets \Call{AttachSparseOutcome}{\mathcal{D}_{\mathrm{on}}^{k}}$
    \State $\widetilde{\mathcal{D}}_{\mathrm{on}}^{V,k} \gets \Call{Downsample}_d(\mathcal{D}_{\mathrm{on,TD}}^{k})$ \Comment{retain original starts $t=0,d,2d,\ldots$ while each TD tuple still spans $H$ steps}
    \State $B_{V}^{k} \gets \operatorname{Mix}_{1:1}(\mathcal{D}_{\mathrm{demo}}\cup\mathcal{D}_{\mathrm{off}},\widetilde{\mathcal{D}}_{\mathrm{on}}^{V,k})$
    \State $\widehat{\mathcal{D}}_{\mathrm{on}}^{k} \gets \Call{Annotate}{\mathcal{D}_{\mathrm{on}}^{k},V_{\psi^{(k-1)}}}$
    \State $B_{\pi}^{k} \gets \operatorname{Mix}_{1:1:1}(\mathcal{D}_{\mathrm{demo}},\mathcal{D}_{\mathrm{HI,off}},\widehat{\mathcal{D}}_{\mathrm{on}}^{k})$
    \Statex \hspace{\algorithmicindent}\textbf{Synchronously:}
    \State $V_{\psi^{(k)}} \gets \operatorname{TD}(V_{\psi^{(k-1)}};B_V^k)$
    \State $\pi_{\theta^{(k)}} \gets \operatorname{FM}(\pi_{\theta^{(k-1)}};B_{\pi}^{k})$
    \State $\mathcal{D}_{\mathrm{off}} \gets \mathcal{D}_{\mathrm{off}}\cup\mathcal{D}_{\mathrm{on}}^{k}$
\EndFor
\State \Return $\pi_{\theta^{(N_{\mathrm{D}}+N_{\mathrm{E}})}},V_{\psi^{(N_{\mathrm{D}}+N_{\mathrm{E}})}}$
\end{algorithmic}
\par\nobreak\hrule
\endgroup
\addvspace{\intextsep}

\section{Training Data and Objective Details}
\label{app:training_details}

\subsection{Distributional TD Targets}
\label{app:value_targets}

Gaussian labels are used only for progress initialization in Eq.~\ref{eq:progress_initialization}.
TD learning instead transforms the complete target-network distribution without first taking its expectation.
For each retained $H$-step tuple, $d_t=1$ when the segment reaches a task outcome or the terminal takeover boundary of an intervention prefix, and $d_t=0$ otherwise.
In a complete rollout, a human takeover does not terminate the backup, allowing subsequent recovery to contribute to the target.
We define the transformed atom $y_{t,k}=R_{t:t+H}+\gamma^H(1-d_t)z_k$ and its support-clipped location $\widetilde y_{t,k}=\min(1,\max(0,y_{t,k}))$.
With atom spacing $\Delta z=1/(K-1)$ and $(u)_+=\max(u,0)$, the linear categorical projection is
\begin{equation}
    q_{t,i}^{\mathrm{TD}}
    =\sum_{k=1}^{K}p_{\bar\psi}(z_k\mid s_{t+H})
    \left(1-\frac{|z_i-\widetilde y_{t,k}|}{\Delta z}\right)_+,
    \qquad i=1,\ldots,K.
    \label{eq:categorical_projection}
\end{equation}
Each transformed atom distributes its mass between its two neighboring support atoms, with all mass assigned to one atom when their locations coincide.
At termination, all bootstrap atoms map to the accumulated reward, independent of the next-state prediction.
The scalar advantage uses the same reward and terminal mask but the online value network $V_\psi$.

\subsection{Warmup Value-Data Construction}

We deploy $(\pi_{\theta^{(0)}},V_{\psi^{(0)}})$ before online RL to collect policy rollouts and human-intervention trajectories, denoted by $\mathcal{D}_{\mathrm{warm}}^{\mathrm{on}}=\mathcal{D}_{\pi_{\theta^{(0)}}}^{\mathrm{rollout}}\cup\mathcal{D}_{\pi_{\theta^{(0)}}}^{\mathrm{HI}}$, where the two subsets contain policy-controlled rollouts and human-controlled segments, respectively.
For a trajectory $\tau_{0:T}\in\mathcal{D}_{\mathrm{warm}}^{\mathrm{on}}$ with $J$ human-intervention segments, we locate their start indices $0<t_0<\cdots<t_{J-1}<T$ and construct the prefix set
\begin{equation}
    \mathcal{P}(\tau_{0:T})
    =\left\{\tau_{0:t_j}\right\}_{j=0}^{J-1}
    \cup\left\{\tau_{0:T}\right\}.
    \label{eq:intervention_prefixes}
\end{equation}
GR-RL constructs failed prefixes at annotated retry keyframes while retaining the original successful trajectory \citep{li2025grrl}.
We use human-takeover boundaries instead: each prefix $\tau_{0:t_j}$ has zero rewards and terminates at $t_j$, whereas the complete rollout $\tau_{0:T}$ retains its outcome reward and continues bootstrapping across takeovers.
This augmentation supplies unsuccessful-attempt supervision alongside recovery outcomes to make the value expert sensitive to retries and partial failures.
Shared observations can therefore receive different TD targets under the two termination conventions.
The categorical predictor can represent mixtures of these targets, including multimodal distributions, under the cross-entropy objective.
Its expectation reflects the augmented training data rather than solely the original task's uninterrupted return.

Temporal downsampling selects trajectory start indices at integer stride $d\geq1$ while preserving each $H$-step reward and its terminal-masked bootstrap at $s_{t+H}$.
For a trajectory $\tau_{0:L}$, the operator $\operatorname{DS}_d$ selects
\begin{equation}
    \operatorname{DS}_d(\tau_{0:L})
    =\left\{
    \bigl(s_t,a_{t:t+H},R_{t:t+H},s_{t+H},d_t\bigr)
    :t\in\{0,d,2d,\ldots\},\ t+H\leq L
    \right\}.
    \label{eq:temporal_downsampling}
\end{equation}
For a set of prefixes, $\operatorname{DS}_d$ is applied to each prefix and the resulting transition sets are unioned.
We construct the processed warmup rollout dataset for value learning as
\begin{equation}
    \widetilde{\mathcal{D}}_{\mathrm{warm}}^V
    =\bigcup_{\tau\in\mathcal{D}_{\mathrm{warm}}^{\mathrm{on}}}
    \ \bigcup_{\tau'\in\mathcal{P}(\tau)}
    \operatorname{DS}_d(\tau').
    \label{eq:warmup_value_dataset}
\end{equation}
Algorithm~\ref{alg:training_pipeline} specifies the corresponding warmup and online batch sources, fixed mixing ratios, annotation snapshots, and buffer updates.

\subsection{Online Data Sources}

We initialize the online-round indexing by $(\pi_{\theta^{(0)}},V_{\psi^{(0)}})=(\pi_{\theta_{\mathrm{warm}}},V_{\psi_{\mathrm{warm}}})$ and use $k\in\{1,\ldots,N_{\mathrm{D}}\}$ to index the $N_{\mathrm{D}}$ Online DAgger rounds.
The offline source is the fixed demonstration dataset $\mathcal{D}_{\mathrm{demo}}$.
The on-policy source, denoted by $\mathcal{D}_{\mathrm{on}}^k$, consists of the current rollouts collected by deploying $\pi_{\theta^{(k-1)}}$ with human intervention enabled.
The off-policy source contains all rollouts collected before round $k$:
\begin{equation}
    \mathcal{D}_{\mathrm{off}}^k
    =\mathcal{D}_{\mathrm{warm}}^{\mathrm{on}}
    \cup\bigcup_{i=1}^{k-1}\mathcal{D}_{\mathrm{on}}^i.
    \label{eq:online_buffer_definitions}
\end{equation}
We use $\operatorname{HI}(\mathcal{D})$ to extract all human-controlled segments from a trajectory set, giving the historical correction data $\mathcal{D}_{\mathrm{HI,off}}^k=\operatorname{HI}(\mathcal{D}_{\mathrm{off}}^k)$.
For value learning, the offline and off-policy sources form the replay dataset $\mathcal{D}_{\mathrm{replay}}^k=\mathcal{D}_{\mathrm{demo}}\cup\mathcal{D}_{\mathrm{off}}^k$.
We process $\mathcal{D}_{\mathrm{on}}^k$ using the intervention-prefix construction and temporal downsampling, yielding $\widetilde{\mathcal{D}}_{\mathrm{on}}^{V,k}$.
We use $\operatorname{Annotate}(\mathcal{D};V_\psi)$ to attach the value-derived advantage conditions produced by $V_\psi$ to the samples in $\mathcal{D}$.

\section{Formal Setup for Soft-Aliasing Analysis}
\label{app:soft_aliasing_setup}

We assume that action sequences lie in the Euclidean product space $\mathcal{A}^{H}$, and we use $\mathcal{P}_2(\mathcal{A}^{H})$ to denote the probability distributions with finite second moment, $\int_{\mathcal{A}^{H}}\|a\|_2^2\,\mathrm{d}\mu(a)<\infty$.
This condition only ensures that the Wasserstein distance $W_2$ is finite.
We let $h_t$ denote the full interaction history at timestep $t$.
For each analyzed state $i\in\{1,2\}$, $\mathcal{H}_i$ denotes its corresponding local set of full histories among samples labeled $o_t=1$ as defined in Sec.~\ref{sec:advantage_conditioned_policy}.
Their target action-chunk distributions are
\begin{equation}
    \mu_i
    =\mathcal{L}(a_{t:t+H}\mid h_t\in\mathcal{H}_i,o_t=1)
    \in\mathcal{P}_2(\mathcal{A}^{H}),
    \qquad i\in\{1,2\},
    \label{eq:appendix_target_action_law}
\end{equation}
where $\mathcal{L}$ denotes the probability law under the effective policy-training distribution.
For representation type $r\in\{\mathrm{obs},\pi\}$, we use the representations $z_i^r$ introduced in Sec.~\ref{sec:soft_aliasing_analysis} and their corresponding distances $\varepsilon_r$ in Eq.~\ref{eq:value_cache_geometry}.
We define $q_i^r=\pi_\theta(\cdot\mid z_i^r,o_t=1)$ and the corresponding policy prediction error $e_i^r=W_2(q_i^r,\mu_i)$.
The local $L$-Lipschitz assumption used in Theorem~\ref{thm:soft_aliasing} is
\begin{equation}
    W_2(q_1^r,q_2^r)\leq L\varepsilon_r.
    \label{eq:appendix_local_lipschitz}
\end{equation}
If the two states have relative weights $\alpha$ and $1-\alpha$, where $\alpha\in(0,1)$, the same argument also yields the frequency-weighted bound
\begin{equation}
    \alpha(e_1^r)^2+(1-\alpha)(e_2^r)^2
    \geq
    \alpha(1-\alpha)(\Delta_a-L\varepsilon_r)_+^2.
    \label{eq:appendix_weighted_soft_aliasing_bound}
\end{equation}

\section{Proof of the Soft-Aliasing Bound}
\label{app:soft_aliasing_proof}

\begin{proof}[Proof of Theorem~\ref{thm:soft_aliasing}]
We fix $r\in\{\mathrm{obs},\pi\}$ and omit the superscript $r$ from $q_i^r$ and $e_i^r$ for readability.
By the Wasserstein triangle inequality and the local Lipschitz assumption,
\begin{equation}
    \Delta_a
    \leq W_2(\mu_1,q_1)+W_2(q_1,q_2)+W_2(q_2,\mu_2)
    \leq e_1+L\varepsilon_r+e_2.
    \label{eq:appendix_aliasing_triangle}
\end{equation}
Hence $e_1+e_2\geq(\Delta_a-L\varepsilon_r)_+$.
Since $\max_i e_i\geq(e_1+e_2)/2$, this proves Eq.~\ref{eq:soft_aliasing_bound}.
Exact matching on the two inputs requires $q_i=\mu_i$ and therefore requires $\Delta_a\leq L\varepsilon_r$ by Eq.~\ref{eq:appendix_local_lipschitz}.
Consequently, $L\varepsilon_{\mathrm{obs}}<\Delta_a$ rules out exact observation-only matching, whereas $\Delta_a\leq L\varepsilon_\pi$ makes exact matching at $z_1^\pi$ and $z_2^\pi$ compatible with the local Lipschitz constraint.
This compatibility statement does not assert that a particular parameterized policy will attain the targets after optimization.

For completeness, we prove the frequency-weighted extension in Eq.~\ref{eq:appendix_weighted_soft_aliasing_bound}.
For $e_1,e_2\geq0$, direct expansion gives
\begin{equation}
    \alpha e_1^2+(1-\alpha)e_2^2
    -\alpha(1-\alpha)(e_1+e_2)^2
    =\left(\alpha e_1-(1-\alpha)e_2\right)^2\geq0.
    \label{eq:appendix_weighted_aliasing}
\end{equation}
Combining this inequality with $e_1+e_2\geq(\Delta_a-L\varepsilon_r)_+$ proves the weighted bound.
\end{proof}

\end{document}